\documentclass[11pt]{article}

\pdfoutput=1

\usepackage{url}
\usepackage{fullpage}
\usepackage{amsmath,amsthm}
\usepackage{euler}
\usepackage{amssymb}
\usepackage{amscd}
\usepackage{latexsym}
\usepackage{booktabs}
\usepackage{tabularx}
\usepackage{graphicx}   
\usepackage{subfigure}
\usepackage{enumerate}
\usepackage[round]{natbib}

\newcommand{\ie}{\textit{i.e.}\hspace{4pt}}

\newcommand{\cf}{\textit{cf.}\hspace{4pt}}

\newtheorem{theorem}{Theorem}
\newtheorem{definition}{Definition}
\newtheorem{lemma}{Lemma}
\newtheorem{proposition}{Proposition}
\newtheorem{corollary}{Corollary}

\newtheorem{remark}{Remark}

\def\RR{\mathbb R}

\newcommand{\X}{{\cal X}}
\newcommand{\D}{{\cal D}}

\def\D{{\cal D}}

\usepackage{graphicx} 
\usepackage{amssymb}
\usepackage{amsmath}
\usepackage{booktabs}

\def\argmin{\mathop{\rm arg\, min}}
\def\R{\mathbb{R}}

\def\E{\mathbb E}
\def\PROB{{\mathbb P}}
\def\Pcal{{\cal P}}
\def\var{{\rm Var}}

\def\I{{\mathbb I}}
\def\pr{\PROB}

\newcommand{\roc}{\rm  ROC}

\newcommand{\Y}{{\cal Y}}

\def\vus{{\rm VUS}}
\newtheorem{assumption}{Assumption}

\begin{document}

\title{Scaling-up Empirical Risk Minimization:\\Optimization of Incomplete $U$-statistics}

\author{
St\'ephan Cl\'emen\c{c}on\thanks{LTCI, CNRS, T\'el\'ecom ParisTech, Universit\'e Paris-Saclay, 75013, Paris, France. \texttt{\{stephan.clemencon,igor.colin\}@telecom-paristech.fr}}~, 
Aur\'elien Bellet\thanks{Magnet Team, INRIA Lille -- Nord Europe, 59650 Villeneuve d'Ascq, France. \texttt{aurelien.bellet@inria.fr}}~, 
Igor Colin\footnotemark[1]
}

\maketitle


\begin{abstract} In a wide range of statistical learning problems such as ranking, clustering or metric learning among others, the risk is accurately estimated by $U$-statistics of degree $d\geq 1$, \emph{i.e.} functionals of the training data with low variance that take the form of averages over $k$-tuples. From a computational perspective, the calculation of such statistics is highly expensive even for a moderate sample size $n$, as it requires averaging $O(n^d)$ terms. This makes learning procedures relying on the optimization of such data functionals hardly feasible in practice. It is the major goal of this paper to show that, strikingly, such empirical risks can be replaced by drastically computationally simpler Monte-Carlo estimates based on $O(n)$ terms only, usually referred to as \textit{incomplete $U$-statistics}, without damaging the $O_{\mathbb{P}}(1/\sqrt{n})$ learning rate of \textit{Empirical Risk Minimization} (ERM) procedures. For this purpose, we establish uniform deviation results describing the error made when approximating a $U$-process by its incomplete version under appropriate complexity assumptions. Extensions to model selection, fast rate situations and various sampling techniques are also considered, as well as an application to stochastic gradient descent for ERM. Finally, numerical examples are displayed in order to provide strong empirical evidence that the approach we promote largely surpasses more naive subsampling techniques.
\end{abstract}
 

\section{Introduction}

In classification/regression, empirical risk estimates are sample mean statistics and the theory of \textit{Empirical Risk Minimization} (ERM) has been originally developed in this context, see \cite{DGL96}. The ERM theory essentially relies on the study of maximal deviations between these empirical averages and their expectations, under adequate complexity assumptions on the set of prediction rule candidates. The relevant tools are mainly concentration inequalities for empirical processes, see \cite{LT91} for instance.

In a wide variety of problems that received a good deal of attention in the machine learning literature and ranging from clustering to image recognition through ranking or learning on graphs, natural estimates of the risk are not basic sample means but take the form of averages of $d$-tuples, usually referred to as $U$-statistics in Probability and Statistics, see \cite{Lee90}. In \cite{CLV05} for instance, ranking is viewed as pairwise classification and the empirical ranking error of any given prediction rule is a $U$-statistic of order $2$, just like the \textit{within cluster point scatter} in cluster analysis \citep[see][]{CLEM14} or empirical performance measures in metric learning, refer to \cite{Cao2012a} for instance. Because empirical functionals are computed by averaging over tuples of sampling observations, they exhibit a complex dependence structure, which appears as the price to be paid for low variance estimates. \textit{Linearization techniques} \citep[see][]{Hoeffding48} are the main ingredient in studying the behavior of empirical risk minimizers in this setting, allowing to establish probabilistic upper bounds for the maximal deviation of collection of centered $U$-statistics under appropriate conditions by reducing the analysis to that of standard empirical processes. However, while the ERM theory based on minimization of $U$-statistics is now consolidated \citep[see][]{CLV08}, putting this approach in practice generally leads to significant computational difficulties that are not sufficiently well documented in the machine learning literature. In many concrete cases, the mere computation of the risk involves a summation over an extremely high number of tuples and runs out of time or memory on most machines.  

Whereas the availability of massive information in the Big Data era, which machine learning procedures could theoretically now rely on, has motivated the recent development of \textit{parallelized / distributed} approaches in order to scale-up certain statistical learning algorithms, see \cite{BBL11} or \cite{BCJM13} and the references therein, the present paper proposes to use \textit{sampling techniques} as a remedy to the apparent intractability of learning from data sets of explosive size, in order to break the current computational barriers. More precisely, it is the major goal of this article to study how a simplistic sampling technique (\textit{i.e.} drawing with replacement) applied to risk estimation, as originally proposed by \cite{Blom76} in the context of asymptotic pointwise estimation, may efficiently remedy this issue without damaging too much the ``reduced variance'' property of the estimates, while preserving the learning rates (including certain "fast-rate" situations).
For this purpose, we investigate to which extent a $U$-process, that is a collection of $U$-statistics, can be accurately approximated by a Monte-Carlo version (which shall be referred to as an \textit{incomplete $U$-process} throughout the paper) involving much less terms, provided it is indexed by a class of kernels of controlled complexity (in a sense that will be explained later). A maximal deviation inequality connecting the accuracy of the approximation to the number of terms involved in the approximant is thus established. This result is the key to the analysis of the statistical performance of minimizers of risk estimates when they are in the form of an incomplete $U$-statistic. In particular, this allows us to show the advantage of using this specific sampling technique, compared to more naive approaches with exactly the same computational cost, consisting for instance in first drawing a subsample and then computing a risk estimate of the form of a (complete) $U$-statistic based on it. We also show how to incorporate this sampling strategy into iterative statistical learning techniques based on stochastic gradient descent (SGD), see \cite{B98}. The variant of the SGD method we propose involves the computation of an incomplete $U$-statistic to estimate the gradient at each step.  For the estimator thus produced, rate bounds describing its statistical performance are established under mild assumptions. Beyond theoretical results, we present illustrative numerical experiments on metric learning and clustering with synthetic and real-world data that support the relevance of our approach. 


The rest of the article is organized as follows. 
In Section~\ref{sec:background}, we recall basic definitions and concepts pertaining to the theory of $U$-statistics/processes and present important examples in machine learning where natural estimates of the performance/risk measure are $U$-statistics. We then review the existing results for the empirical minimization of complete $U$-statistics.
In Section \ref{sec:approx}, we recall the notion of incomplete $U$-statistic and we derive maximal deviation inequalities describing the error made when approximating a $U$-statistic by its incomplete counterpart uniformly over a class of kernels that fulfills appropriate complexity assumptions. This result is next applied to derive (possibly fast) learning rates for minimizers of the incomplete version of the empirical risk and to model selection. Extensions to incomplete $U$-statistics built by means of other sampling schemes than sampling with replacement are also investigated. In Section \ref{sec:SGD}, estimation by means of incomplete $U$-statistics is applied to stochastic gradient descent for iterative ERM. Section~\ref{sec:exp} presents some numerical experiments. Finally, Section~\ref{sec:conclu} collects some concluding remarks.
Technical details are deferred to the Appendix.

\section{Background and Preliminaries}\label{sec:background}

As a first go, we briefly recall some key notions of the theory of $U$-statistics (Section~\ref{subsec:Ustat}) and provide several examples of statistical learning problems for which natural estimates of the performance/risk measure are in the form of $U$-statistics (Section~\ref{subsec:examples}). Finally, we review and extend the existing rate bound analysis for the empirical minimization of (complete) generalized $U$-statistics (Section~\ref{subsec:erm_complete}). Here and throughout, $\mathbb{N}^*$ denotes the set of all strictly positive integers, $\RR_+$ the set of nonnegative real numbers. 

\subsection{$U$-Statistics/Processes: Definitions and Properties}\label{subsec:Ustat}

For clarity, we recall the definition of generalized $U$%
-statistics.
An excellent account of properties and asymptotic theory of $U$-statistics can be found in \cite{Lee90}. 
\begin{definition}\label{def:Ustat}{\sc (Generalized $U$-statistic)}
Let $K\geq 1$ and $(d_1,\; \ldots,\; d_K)\in \mathbb{N}^{*K}$. Let $\mathbf{X}_{\{1,\;\ldots,\; n_k  \}}=(X^{(k)}_{1},\;\ldots,\; X^{(k)}_{n_k})$, $1\leq k\leq K$, be $K$ independent samples of sizes $n_k\geq d_k$ and composed of i.i.d. random variables taking their values in some measurable space $\X_k$ with distribution $F_k(dx)$ respectively. 
Let  $H:\X_1^{d_1}\times \cdots \times \X_K^{d_K}\rightarrow\mathbb{R}$ be a measurable function, square integrable with respect to the probability distribution $\mu=F_1^{\otimes d_1}\otimes \cdots \otimes F_K^{\otimes d_K}$. Assume in addition (without loss of generality) that $H(\mathbf{x}^{(1)},\; \ldots, \; \mathbf{x}^{(K)})$ is symmetric within each block of arguments $\mathbf{x}^{(k)}$ (valued in $\X^{d_k}_k$), $1\leq k\leq K$.
The generalized (or $K$-sample) $U$-statistic of degrees $(d_1,\; \ldots,\; d_K)$ with kernel $H$, is then defined as
\begin{equation}\label{eq:UstatG}
U_{\mathbf{n}}(H)=\frac{1}{\prod_{k=1}^K \binom{n_k}{d_k}}\sum_{I_1}\ldots\sum_{I_K} H(\mathbf{X}^{(1)}_{I_1},\; \mathbf{X}^{(2)}_{I_2},\; \ldots,\; \mathbf{X}^{(K)}_{I_K}),
\end{equation}
 where the symbol $\sum_{I_k}$ refers to summation over all $\binom{n_k}{d_k}$ subsets $\mathbf{X}^{(k)}_{I_k}=( X^{(k)}_{i_1},\;\ldots,\; X^{(k)}_{i_{d_k}})$ related to a set $I_k$ of $d_k$ indexes $1\leq i_1< \ldots <i_{d_k}\leq n_k$ and $\mathbf{n}=(n_1,\; \ldots,\; n_K)$.
\end{definition}

The above definition generalizes standard sample mean statistics, which correspond to the case $K=1=d_1$. More generally when $K=1$, $U_{\mathbf{n}}(H)$ is an average over all $d_1$-tuples of observations, while $K\geq 2$ corresponds to the multi-sample situation with a $d_k$-tuple for each sample $k\in\{1,\dots,K\}$. A $U$-process is defined as a collection of $U$-statistics indexed by a set $\mathcal{H}$ of kernels. This concept generalizes the notion of empirical process.

Many statistics used for pointwise estimation or
hypothesis testing are actually generalized $U$-statistics (\textit{e.g.} the sample
variance, the Gini mean difference, the Wilcoxon Mann-Whitney statistic,
Kendall tau). Their popularity mainly arises from their ``reduced variance''
property: the statistic $U_{\mathbf{n}}(H)$ has minimum variance among all
unbiased estimators of the parameter 
\begin{eqnarray}\label{eq:parameter}
\mu(H)&=&\mathbb{E}\left[H(X^{(1)}_{1},\;\ldots,\; X^{(1)}_{d_1},\; \ldots,\;
X^{(K)}_{1},\;\ldots,\; X^{(K)}_{d_K})\right]\\
&=& \int_{\mathbf{x}^{(1)}\in \X_1^{d_1}}\cdots  \int_{\mathbf{x}^{(K)}\in \X_K^{d_K}}  H(\mathbf{x}^{(1)},\; \ldots, \; \mathbf{x}^{(K)}) dF_1^{\otimes d_1}(\mathbf{x}^{(1)})\cdots dF_K^{\otimes d_K}(\mathbf{x}^{(K)})=\mathbb{E}\left[U_{\mathbf{n}}(H)  \right].\nonumber
\end{eqnarray}

Classically, the limit properties of these statistics
(law of large numbers, central limit theorem, \textit{etc.}) are investigated in an asymptotic framework
stipulating that, as the size of the full pooled sample 
\begin{equation}\label{eq:fullsize}
n\overset{def}{=}n_1+\;\ldots+n_K
\end{equation}
tends to infinity, we have: 
\begin{equation}\label{asymptotics}
n_k/n\rightarrow \lambda_k>0 \text{ for }
k=1,\;\ldots,\;K.
\end{equation} 
Asymptotic results and deviation/moment inequalities for $K$-sample $U$-statistics can be classically established by means of specific representations of this class of functionals, see \eqref{eq:Hoeffding} and \eqref{eq:Hajek} introduced in later sections. Significant progress in the analysis of $U$-statistics and $U$-processes has then recently been achieved by means of decoupling theory, see \cite{PenaGine99}.
For completeness, we point out that the asymptotic behavior of (multisample) $U$-statistics has been investigated under weaker integrability assumptions than that stipulated in Definition \ref{def:Ustat}, see \cite{Lee90}.

\subsection{Motivating Examples}\label{subsec:examples}

In this section, we review important supervised and unsupervised statistical learning problems where the empirical performance/risk measure is of the form of a generalized $U$-statistics. They shall serve as running examples throughout the paper.

\subsubsection{Clustering}\label{sec:clustering}

Clustering refers to the unsupervised learning task that consists in partitioning a set of data points $X_1,\; \ldots,\; X_n$ in a feature space $\mathcal{X}$ into a finite collection of subgroups depending on their similarity (in a sense that must be specified): roughly, data points in the same subgroup should be more similar to each other than to those lying in other subgroups. One may refer to Chapter 14 in \cite{FriedHasTib09} for an account of state-of-the-art clustering techniques. 
 Formally, let $M\geq 2$ be the number of desired clusters and consider a symmetric function $D: \X\times\X\rightarrow \R_+$ such that $D(x,x)=0$ for any $x\in \mathcal{X}$. $D$ measures the dissimilarity between pairs of observations $(x,x')\in \X^2$: the larger $D(x,x')$, the less similar $x$ and $x'$. For instance, if $\mathcal{X}\subset\R^d$, $D$ could take the form $D(x,x')=\Psi(\|x-x'\|_q)$, where $q\geq 1$, $\vert\vert a\vert\vert_q=(\sum_{i=1}^d\vert a_i\vert^q)^{1/q}$ for all $a\in \RR^d$ and $\Psi:\RR_+\rightarrow\RR_+$ is any borelian nondecreasing function such that $\Psi(0)=0$.
 In this context, the goal of clustering methods is to find a partition $\Pcal$ of the feature space $\X$ in a class $\Pi$ of partition candidates that minimizes the following \textit{empirical clustering risk}:
\begin{equation}\label{eq:emp_clust_risk}
\widehat{W}_{n}(\Pcal)=\frac{2}{n(n-1)}\sum_{1\leq i<j \leq n}D(X_{i},X_{j})\cdot\Phi_{\mathcal{P}}(X_{i},X_{j}),
\end{equation}
where $\Phi_{\mathcal{P}}(x,x')=\sum_{\mathcal{C}\in \mathcal{P}}\mathbb{I}\{(x,x')\in \mathcal{C}^2  \}$.
Assuming that the data $X_1,\; \ldots,\; X_n$ are i.i.d. realizations of a generic random variable $X$ drawn from an unknown probability distribution $F(dx)$ on $\mathcal{X}$, the quantity $\widehat{W}_{n}(\Pcal)$, also known as the \textit{intra-cluster similarity} or \textit{within cluster point scatter}, is a one sample $U$-statistic of degree two ($K=1$ and $d_1=2$) with kernel given by:
\begin{equation}\label{eq:kernel_clust} \forall (x,x')\in \X^2,\;\; 
H_{\mathcal{P}}(x,x')= D(x,x')\cdot \Phi_{\mathcal{P}}(x,x'),
\end{equation}
according to Definition \ref{def:Ustat} provided that $\int\int_{(x,x')\in \X^2}D^2(x,x')\cdot \Phi_{\Pcal}(x,x')F(dx)F(dx')<+\infty$.
The expectation of the empirical clustering risk $\widehat{W}_{n}(\Pcal)$  is given by
\begin{equation}\label{eq:clust_risk}
W(\Pcal)=\E\left[D(X,X')\cdot \Phi_{\Pcal}(X,X')\right],
\end{equation}
where $X'$ is an independent copy of the r.v. $X$, and is named the \textit{clustering risk} of the partition $\mathcal{P}$.
The statistical analysis of the clustering performance of minimizers $\widehat{\Pcal}_n$ of the empirical risk  \eqref{eq:emp_clust_risk} over a class $\Pi$ of appropriate complexity can be found in \cite{CLEM14}. Based on the theory of $U$-processes, it is shown in particular how to establish rate bounds for the excess of clustering risk of any empirical minimizer, $W(\widehat{\Pcal}_n)-\inf_{\Pcal\in \Pi}W(\Pcal)$ namely, under appropriate complexity assumptions on the cells forming the partition candidates.

\subsubsection{Metric Learning}\label{sec:metriclearning}

Many problems in machine learning, data mining and pattern recognition (such as the clustering problem described above) rely on a metric to measure the distance between data points. Choosing an appropriate metric for the problem at hand is crucial to the performance of these methods. Motivated by a variety of applications ranging from computer vision to information retrieval through bioinformatics, metric learning aims at adapting the metric to the data and has attracted a lot of interest in recent years \citep[see for instance][for an account of metric learning and its applications]{BHS14}.
As an illustration, we consider the metric learning problem for supervised classification. In this setting, we observe independent copies $(X_{1},Y_{1}),\ldots,(X_{n},Y_{n})$ of a random couple $(X,Y)$, where the r.v. $X$ takes values in some feature space $\X$ and $Y$ in a finite set of labels, $\Y=\{1,\; \ldots,\; C\}$ with $C\geq 2$ say. Consider a set $\mathcal{D}$ of distance measures $D: \X\times\X\rightarrow \R_+$. Roughly speaking, the goal of metric learning in this context is to find a metric under which pairs of points with the same label are close to each other and those with different labels are far away. The risk of a metric $D$ can be expressed as:
\begin{equation}
\label{eq:metricgen}
R(D)= \E\left[\phi\left((1-D(X,X')\cdot (2\mathbb{I}\{Y=Y' \}-1)\right)\right],
\end{equation}
where $\phi(u)$ is a convex loss function upper bounding the indicator function $\mathbb{I}\{u\geq 0  \}$, such as the hinge loss $\phi(u)=\max(0,1-u)$. The natural empirical estimator of this risk is 
\begin{equation}
\label{eq:metricemp}
R_{n}(D)= \frac{2}{n(n-1)}\sum_{1\leq i<j\leq n}\phi\left((D(X_{i},X_{j})-1)\cdot (2\mathbb{I}\{ Y_{i}=Y_{j}\} -1 )\right),
\end{equation}
which is a one sample $U$-statistic of degree two with kernel given by:
\begin{equation}
H_D\left((x,y),(x',y)  \right)=\phi\left((D(x,x')-1)\cdot (2\mathbb{I}\{ y=y' \} -1 )  \right).
\end{equation} 

The convergence to \eqref{eq:metricgen} of a minimizer of \eqref{eq:metricemp} has been studied in the frameworks of algorithmic stability \citep{Jin2009a}, algorithmic robustness \citep{Bellet2015a} and based on the theory of $U$-processes under appropriate regularization \citep{Cao2012a}.



\subsubsection{Multipartite Ranking}

Given objects described by a random vector of attributes/features $X\in \mathcal{X}$ and the (temporarily hidden) ordinal labels $Y\in\{1,\; \ldots,\; K\}$ assigned to it, the goal of \textit{multipartite ranking} is to rank them in the same order as that induced by the labels, on the basis of a training set of labeled examples.  This statistical learning problem finds many applications in a wide range of fields (\textit{e.g.} medicine, finance, search engines, e-commerce). Rankings are generally defined by means of a scoring function $s:\mathcal{X}\rightarrow \mathbb{R}$, transporting the natural order on the real line onto the feature space and the gold standard for evaluating the ranking performance of $s(x)$ is the $\roc$ manifold, or its usual summary the $\vus$ criterion ($\vus$ standing for \textit{Volume Under the $\roc$ Surface}), see \cite{ClemRob14} and the references therein.  In \cite{CRV13}, optimal scoring functions have been characterized as those that are optimal for all bipartite subproblems. In other words, they are increasing transforms of the likelihood ratio $dF_{k+1}/dF_{k}$, where $F_k$ denotes the class-conditional distribution for the $k$-th class. 
When the set of optimal scoring functions is non-empty, the authors also showed that it corresponds to the functions which maximize the volume under the $\roc$ surface 
$$
VUS(s)=\pr\{s(X_{1})<\ldots<s(X_{K})|Y_{1}=1,\; \ldots,\;Y_{K}=K\}.    
$$
Given $K$ independent samples $(X_{1}^{(k)},\ldots,X_{n_{k}}^{(k)})\overset{i.i.d.}{\sim}F_k(dx)$ for $k=1,\; \ldots,\; K$, the empirical counterpart of the VUS can be written in the following way: 
\begin{equation}\label{eq:emp_vus}
\widehat{VUS}(s)=\frac{1}{\prod^{K}_{k=1}n_{k}}\sum^{n_{1}}_{i_{1}=1}\ldots\sum^{n_{K}}_{i_{K}=1}\I\{s(X^{(1)}_{i_{1}})<\ldots<s(X^{(K)}_{i_{K}})\}.
\end{equation}
The empirical $\vus$ \eqref{eq:emp_vus} is a $K$-sample $U$-statistic of degree $(1,\; \ldots,\; 1)$ with kernel given by:
\begin{equation}
H_s(x_1,\; \ldots,\; x_K)=\mathbb{I}\{s(x_{1})<\ldots<s(x_{K})   \}.
\end{equation}

\subsection{Empirical Minimization of $U$-Statistics}
\label{subsec:erm_complete}

As illustrated by the examples above, many learning problems can be formulated as finding a certain rule $g$ in a class $\mathcal{G}$ in order to minimize a risk of the same form as \eqref{eq:parameter},
$\mu(H_g)$, with kernel $H=H_g$. Based on $K\geq 1$ independent i.i.d. samples 
$$
\mathbf{X}^{(k)}_{\{1,\;\ldots,\; n_k  \}}=(X^{(k)}_{1},\;\ldots,\; X^{(k)}_{n_k}) \text{ with } 1\leq k\leq K,$$
the ERM paradigm in statistical learning suggests to replace the risk by the $U$-statistic estimation $U_{\mathbf{n}}(H_g)$ in the minimization problem. The study of the performance of minimizers $\widehat{g}_{\mathbf{n}}$ of the empirical estimate $U_{\mathbf{n}}(H_g)$ over the class $\mathcal{G}$ of rule candidates naturally leads to analyze the fluctuations of the $U$-process
\begin{equation}
\left\{U_{\mathbf{n}}(H_g)- \mu(H_g):\; g\in \mathcal{G}  \right\}.
\end{equation}
Given the bound
\begin{equation}
\mu(H_{\widehat{g}_{\mathbf{n}}})-\inf_{g\in \mathcal{G}}\mu(H_g)\leq 2 \sup_{g\in \mathcal{G}}\vert U_{\mathbf{n}}(H_g)- \mu(H_g) \vert, 
\end{equation}
a probabilistic control of the maximal deviation $\sup_{g\in \mathcal{G}}\vert U_{\mathbf{n}}(H_g)- \mu(H_g) \vert$ naturally provides statistical guarantees for the generalization ability of the empirical minimizer $\widehat{g}_{\mathbf{n}}$. 
As shown at length in the case $K=1$ and $d_1=2$ in \cite{CLV08} and in \cite{CLEM14} for specific problems, this can be achieved under adequate complexity assumptions of the class $\mathcal{H}_{\mathcal{G}}=\{H_{g}:\; g\in \mathcal{G}  \}$. These results rely on the \textit{Hoeffding's representation} of $U$-statistics, which we recall now for clarity in the general multisample $U$-statistics setting.
Denote by $\mathfrak{S}_m$ the symmetric group of order $m$ for any $m\geq 1$ and by $\sigma(i)$ the $i$-th coordinate of any permutation $\sigma\in\mathfrak{S}_m$ for $1\leq i \leq m$. Let $\lfloor z\rfloor$ be the integer part of any real number $z$ and set 
$$
N=\min \left\{  \lfloor n_1/d_1\rfloor,\; \ldots,\; \lfloor n_K/d_K\rfloor \right\}.
$$
Observe that the $K$-sample $U$-statistic \eqref{eq:UstatG} can be expressed as
\begin{equation}\label{eq:Hoeffding}
U_{\mathbf{n}}(H)=\frac{1}{\prod_{k=1}^K n_k!}\sum_{\sigma_1\in \mathfrak{S}_{n_1}}\cdots \sum_{\sigma_K\in \mathfrak{S}_{n_K}}V_H\left(X^{(1)}_{\sigma_1(1)},\; \ldots,\; X^{(K)}_{\sigma_K(n_K)}\right),
\end{equation}
where
\begin{multline*}
V_H\left(X^{(1)}_1,\; \ldots,\; X^{(1)}_{n_1},\; \ldots,\; X_{1}^{(K)},\; \ldots,\; X_{n_K}^{(K)}   \right)
=\frac{1}{N}\Big[H\left(X^{(1)}_1,\; \ldots,\; X^{(1)}_{d_1},\; \ldots,\; X_{1}^{(K)},\; \ldots,\; X_{d_K}^{(K)}   \right)\\+
H\left(X^{(1)}_{d_1+1},\; \ldots,\; X^{(1)}_{2d_1},\; \ldots,\; X_{d_K+1}^{(K)},\; \ldots,\; X_{2d_K}^{(K)}   \right)+\ldots\\
+ H\left(X^{(1)}_{(N-1)d_1+1},\; \ldots,\; X^{(1)}_{N d_1},\; \ldots,\; X_{(N-1)d_K+1}^{(K)},\; \ldots,\; X_{N d_K}^{(K)}   \right)\Big].
\end{multline*}
This representation, sometimes referred to as the \textit{first Hoeffding's decomposition} \citep[see][]{Hoeffding48}, allows to reduce a first order analysis to the case of sums of i.i.d. random variables. 
The following result extends Corollary 3 in \cite{CLV08} to the multisample situation.
  \begin{proposition}\label{prop:Uproc}Let $\mathcal{H}$ be a collection of bounded symmetric kernels on $\prod_{k=1}^K\mathcal{X}_{k}^{d_k}$ such that
  \begin{equation}\label{cond:unifbound}
  \mathcal{M}_{\mathcal{H}}\overset{def}{=}\sup_{(H,x)\in \mathcal{H}\times\X }\vert H(x) \vert<+\infty.
  \end{equation}
  Suppose also that $\mathcal{H}$
  is a {\sc VC} major class of functions with finite Vapnik-Chervonenkis dimension $V<+\infty$. For all $\delta\in (0,1)$, we have with probability at least $1-\delta$,
   \begin{equation}\label{eq:unif_Uproc}
   \sup_{H\in \mathcal{H}} \left\vert U_{
   \mathbf{n}}(H)-\mu(H) \right\vert\leq \mathcal{M}_{\mathcal{H}}\left\{2\sqrt{\frac{2V\log(1+N)}{N}}+\sqrt{\frac{\log(1/\delta)}{N}}\right\},
   \end{equation}
   where $N=\min \left\{  \lfloor n_1/d_1\rfloor,\; \ldots,\; \lfloor n_K/d_K\rfloor \right\}$.
  \end{proposition}
 Observe that, in the usual asymptotic framework \eqref{asymptotics}, the bound \eqref{eq:unif_Uproc} shows that the learning rate is, as expected, of order $O_{\mathbb{P}}(\sqrt{\log{n}/n})$, where $n$ denotes the size of the pooled sample.
 \begin{remark}{\sc (Uniform boundedness)}
 We point out that condition \eqref{cond:unifbound} is clearly satisfied for the class of kernels considered in the multipartite ranking situation, whatever the class of scoring functions considered. In the case of the clustering example, it is fulfilled as soon as the essential supremum of $D(X,X')\cdot \Phi_{\Pcal}(X,X')$ is uniformly bounded over $\Pcal\in \Pi$, whereas in the metric learning example, it is satisfied when the essential supremum of the r.v. $\phi((D(X,X')-1)\cdot (2\mathbb{I}\{ Y=Y' \}-1))$ is uniformly bounded over $D\in \mathcal{D}$. We underline that this simplifying condition can be easily relaxed and replaced by appropriate tail assumptions for the variables $H(X_1^{(1)},\; \ldots,\; X^{(K)}_{d_K})$, $H\in \mathcal{H}$, combining the arguments of the subsequent analysis with the classical ``truncation trick'' originally introduced in \cite{FukNagaev}.
 \end{remark}
  \begin{remark}\label{rk:comp1}{\sc (Complexity assumptions)}
  Following in the footsteps of \cite{CLV08} which considered $1$-sample $U$-statistics of degree 2, define the Rademacher average 
   \begin{equation}\label{eq:Rad1}
    \mathcal{R}_{N}=\sup_{H\in\mathcal{H}}\frac{1}{N}\left\vert \sum_{l=1}^{N}\epsilon_l
    H\left(X^{(1)}_{(l-1)d_1+1},\;\ldots,\; X^{(1)}_{ld_1},\;\ldots,\; X^{(K)}_{(l-1)d_K+1},\;\ldots,\; X^{(K)}_{ld_K}\right)\right\vert,
    \end{equation}
    where $\epsilon_1,\; \ldots,\; \epsilon_{N}$ are independent Rademacher random variables (random symmetric sign variables), independent from the $X^{(k)}_i$'s. As can be seen by simply examining the proof of Proposition \ref{prop:Uproc} (Appendix~\ref{app:propUproc}), a control of the maximal deviations similar to \eqref{eq:unif_Uproc} relying on this particular complexity measure can be obtained: the first term on the right hand side is then replaced by the expectation of the Rademacher average $\mathbb{E}[\mathcal{R}_{N}]$, up to a constant multiplicative factor. This expected value can be bounded by standard metric entropy techniques and in the case where $\mathcal{H}$ is a {\sc VC} major class of functions of dimension $V$, we have: 
    $$
    \mathbb{E}[\mathcal{R}_N]\leq \mathcal{M}_{\mathcal{H}}\sqrt{\frac{2V\log (N+1)}{N}}.$$
    See Appendix~\ref{app:propUproc} for further details.
  \end{remark}

\section{Empirical Minimization of Incomplete $U$-Statistics}\label{sec:approx}

We have seen in the last section that the empirical minimization of $U$-statistics leads to a learning rate of $O_{\mathbb{P}}(\sqrt{\log{n}/n})$. However, the computational cost required to find the empirical minimizer in practice is generally prohibitive, as the number of terms to be summed up to compute the $U$-statistic \eqref{eq:UstatG} is equal to:
$$
 \binom{n_1}{d_1}\times \cdots \times  \binom{n_K}{d_K}.
$$
In the usual asymptotic framework \eqref{asymptotics}, it is of order $O(n^{d_1+\ldots+d_K})$ as $n\rightarrow+\infty$.
It is the major purpose of this section to show that, in the minimization problem, the $U$-statistic $U_{\mathbf{n}}(H_g)$ can be replaced  by a Monte-Carlo estimation, referred to as an \textit{incomplete $U$-statistic}, whose computation requires to average much less terms, without damaging the learning rate (Section~\ref{subsec:approx}). We further extend these results to model selection (Section~\ref{subsec:modelselect}), fast rates situations (Section ~\ref{subsec:fastrates}) and alternative sampling strategies (Section~\ref{subsec:sampling}).

\subsection{Uniform Approximation of Generalized $U$-Statistics}\label{subsec:approx}
As a remedy to the computational issue mentioned above, the concept of \textit{incomplete generalized $U$%
-statistic} has been introduced in the seminal contribution of \cite{Blom76}. The calculation of such a functional involves a summation over low cardinality subsets of the $\binom{n_k}{d_k}$ $d_k$-tuples of indices, $1\leq k\leq K$, solely. In the simplest formulation,
the subsets of indices are obtained by \textit{sampling independently with replacement}, leading to
the following definition.

\begin{definition}{\sc (Incomplete Generalized $U$-statistic)}
Let $B\geq1$. The incomplete version of the $U$-statistic \eqref{eq:UstatG} based on $B$ terms is defined by:
\begin{equation}\label{eq:UstatI}
\widetilde{U}_{B}(H)=\frac{1}{B}\sum_{I=(I_1,\;\ldots,\, I_K)\in\mathcal{D}_B} H(\mathbf{X}^{(1)}_{I_1},\;\ldots,\; \mathbf{X}^{(K)}_{I_K})=\frac{1}{B}\sum_{I\in \mathcal{D}_B}H(\mathbf{X}_I),
\end{equation}
where $\mathcal{D}_B$ is a set of cardinality $B$ built by sampling with replacement in the set 
\begin{equation}
\Lambda=\{((i^{(1)}_1,\;\ldots,\; i^{(1)}_{d_1}),\; \ldots,\; (i^{(K)}_1,\;\ldots,\; i^{(K)}_{d_K})):\; 1\leq i^{(k)}_1<\ldots<i^{(k)}_{d_k}\leq n_k,\; 1\leq k \leq K\},
\end{equation}
and $\mathbf{X}_I=(\mathbf{X}^{(1)}_{I_1},\; \ldots,\; \mathbf{X}^{(K)}_{I_K})$ for all $I=(I_1,\; \ldots,\; I_K)\in \Lambda$.
\end{definition}%

\begin{figure}[t]
\centering
\includegraphics[width=\textwidth]{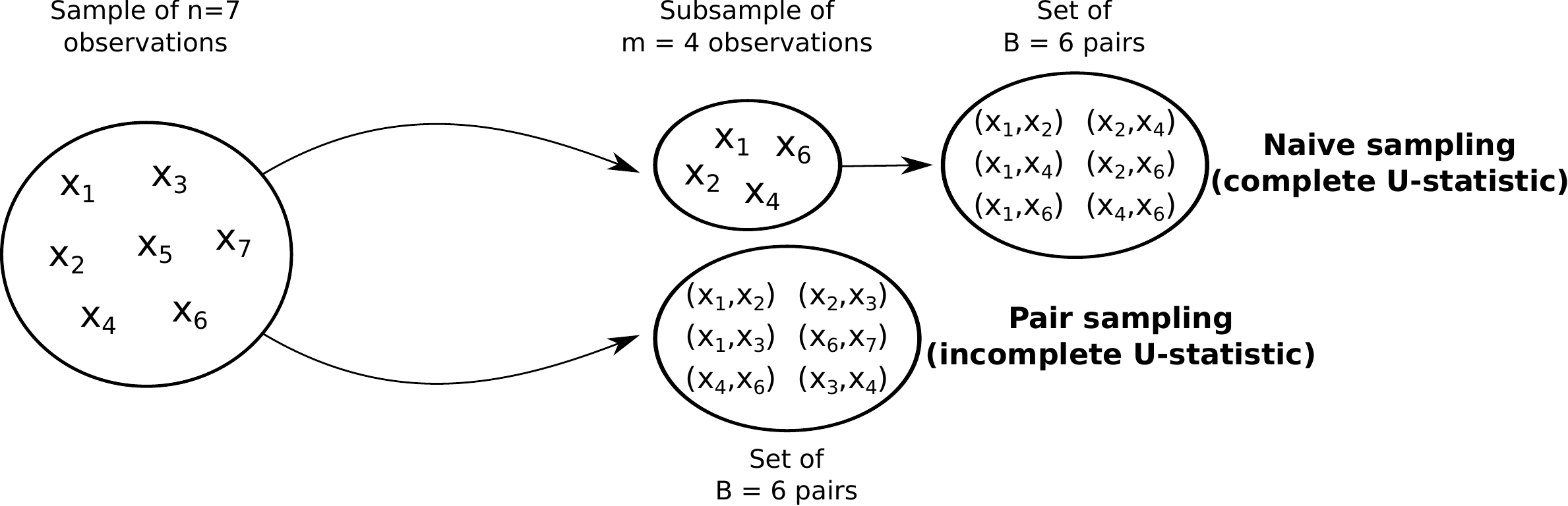}
\caption{Illustration of the difference between an incomplete $U$-statistic and a complete $U$-statistic based on a subsample. For simplicity, we focus on the case $K=1$ and $d_1=2$. In this simplistic example, a sample of $n=7$ observations is considered. To construct a complete $U$-statistic of reduced complexity, we first sample a set of $m=4$ observations and then form all possible pairs from this subsample, \emph{i.e.} $B=m(m-1)/2=6$ pairs in total. In contrast, an incomplete $U$-statistic with the same number of terms is obtained by sampling $B$ pairs directly from the set $\Lambda$ of all possible pairs based on the original statistical population.}
\label{fig:ustat_illustr}
\end{figure}

We stress that the distribution of a complete $U$-statistic built from subsamples of reduced sizes $n'_k$ drawn uniformly at random is quite different from that of an incomplete $U$-statistic based on $B = \prod_{k=1}^K \binom{n'_k}{d_k}$ terms sampled with replacement in $\Lambda$, although they involve the summation of the same number of terms, as depicted by Fig. \ref{fig:ustat_illustr}.

In practice, $B$ should be chosen much smaller than the
cardinality of $\Lambda $, namely $\#\Lambda =\prod_{k=1}^{K}\binom{n_{k}}{%
d_{k}}$, in order to overcome the computational issue previously mentioned.
We emphasize the fact that the cost related to the computation of the value taken by
the kernel $H$ at a given point $(x_{I_{1}}^{(1)},\;\ldots
,\;x_{I_{K}}^{(K)})$ depending on the form of $H$ is not considered here: the focus is on the number of terms involved in the summation solely. As an
estimator of $\mu (H)$, the statistic \eqref{eq:UstatI} is still unbiased, \textit{i.e.} $\mathbb{E}[\widetilde{U}_{B}(H) ]=\mu(H)$,
but its variance is naturally larger than that of the complete $U$-statistic $U_n(H)$.
Precisely, writing the variance of the r.v. $\widetilde{U}_B(H)$ as the expectation of its conditional variance given $(\mathbf{X}_I)_{I\in \Lambda}$ plus the variance of its conditional expectation given $(\mathbf{X}_I)_{I\in \Lambda}$, we obtain
\begin{equation}\label{eq:var_incomp}
\mathrm{Var}(\widetilde{U}_{B}(H))=\left(1-\frac{1}{B}\right)\mathrm{Var}(U_{\mathbf{n}%
}(H))+
\frac{1}{B}\mathrm{Var}(H(X^{(1)}_1,\; \ldots,\; X^{(K)}_{d_K})).
\end{equation}

One may easily check that $\mathrm{Var}(\widetilde{U}_{B}(H))\geq \mathrm{Var}(U_{\mathbf{n}%
}(H))$, and the difference vanishes as $B$ increases. Refer to \cite{Lee90} for further details (see p. 193 therein). Incidentally, we underline that
the empirical variance of \eqref{eq:UstatI} is not easy to compute either
since it involves summing approximately $\#\Lambda $ terms and bootstrap
techniques should be used for this purpose, as proposed in \cite{BerTres06}.
The asymptotic properties of incomplete $U$-statistics have been
investigated in several articles, see \cite{Janson84, BrownKildea78,
Enqvist78}. The angle embraced in the present paper is of very different nature:
the key idea we promote here is to use incomplete versions of collections of 
$U$-statistics in learning problems such as that described in Section~\ref{subsec:examples}. The result stated below shows that this approach solves the numerical problem, while not damaging the learning rates under appropriate complexity assumptions on the collection $\mathcal{H}$ of (symmetric) kernels $H$ considered, the complexity being described here in terms of {\sc VC} dimension for simplicity.
In  particular, it reveals that concentration results established for $U$%
-processes (\textit{i.e.} collections of $U$-statistics) such as Proposition \ref{prop:Uproc} may extend to their
incomplete versions, as shown by the following theorem.

\begin{theorem}\label{thm:main}{\sc (Maximal deviation)} Let $\mathcal{H}$ be a collection of bounded symmetric kernels on $\prod_{k=1}^K\mathcal{X}_{k}^{d_k}$ that fulfills the assumptions of Proposition \ref{prop:Uproc}.
 Then, the following assertions hold true.
\begin{itemize}
\item[(i)]For all $\delta\in (0,1)$, with probability at least $1-\delta$, we have: $\forall \mathbf{n}=(n_1,\; \ldots,\; n_K)\in\mathbb{N}^{*K}$, $\forall B\geq 1$,
\begin{equation*}
 \sup_{H\in \mathcal{H}} \left\vert \widetilde{U}_B(H)-U_{
\mathbf{n}}(H) \right\vert\leq \mathcal{M}_{\mathcal{H}}\times
\sqrt{2\frac{V\log(1+\# \Lambda)+\log (2/\delta)}{B}}
\end{equation*}
\item[(ii)] For all $\delta\in (0,1)$, with probability at least $1-\delta$, we have: $\forall \mathbf{n} \in\mathbb{N}^{*K}$, $\forall B\geq 1$,
\begin{multline*}
\frac{1}{\mathcal{M}_{\mathcal{H}}}\sup_{H\in \mathcal{H}}\left\vert \widetilde{U}_{B}(H)-\mu(H) \right\vert\leq 
2\sqrt{\frac{2 V\log(1+N)}{N}}+\sqrt{\frac{\log(2/\delta)}{N}}\\+\sqrt{2\frac{V\log(1+\#\Lambda)+ \log(4/\delta)}{B}},
\end{multline*}
where $N=\min\{ \lfloor n_1/d_1\rfloor,\ldots,
\lfloor n_K/d_K \rfloor\}$.
\end{itemize}
\end{theorem}

\begin{remark}\label{rk:comp2} {\sc (Complexity assumptions continued)}
We point out that a bound of the same order as that stated above can be obtained under standard metric entropy conditions by means of classical chaining arguments, or under the assumption that the Rademacher average defined by
\begin{equation}\label{eq:Rad2}
\widetilde{\mathcal{R}}_B=\sup_{H\in \mathcal{H}}\frac{1}{B}\left\vert \sum_{b=1}^B\epsilon_b\left\{ \sum_{I\in \Lambda}\zeta_b(I)H(\mathbf{X}_I)\right\} \right\vert
\end{equation}
has an expectation of the order $O(1/\sqrt{B})$. The quantity $\zeta_b(I)$ indicates whether the subset of indexes $I$ has been picked at the $b$-th draw ($\zeta_b(I)=+1$) or not ($\zeta_b(I)=0$), see the calculation at the end of Appendix~\ref{subsec:proofcor}. Equipped with this notation, notice that the $\zeta_b$'s are i.i.d. multinomial random variables such that $\sum_{I\in \Lambda}\zeta_b(I)=+1$. This assumption can be easily shown to be fulfilled in the case where $\mathcal{H}$ is a {\sc VC } major class of finite {\sc VC} dimension (see the proof of Theorem \ref{thm:main} in Appendix~\ref{app:thmmain}). Notice however that although the variables $ \sum_{I\in \Lambda}\zeta_b(I)H(\mathbf{X}_I)$, $1\leq b\leq B$, are conditionally i.i.d. given $(\mathbf{X}_I)_{I\in \Lambda}$, they are not independent and the quantity \eqref{eq:Rad2} cannot be related to complexity measures of the type \eqref{eq:Rad1} mentioned in Remark \ref{rk:comp1}. 
\end{remark}
\begin{remark}
We underline that, whereas $\sup_{H\in \mathcal{H}}\vert U_{\mathbf{n}}(H)-\mu(H) \vert$ can be proved to be of order $O_{\mathbb{P}}(1/n)$ under adequate complexity assumptions in the specific situation where $\{U_{\mathbf{n}}(H):\; H\in \mathcal{H} \}$ is a collection of degenerate $U$-statistics (see Section \ref{subsec:fastrates}), the bound $(i)$ in Theorem \ref{thm:main} cannot be improved in the degenerate case. Observe indeed that, conditioned upon the observations $X^{(k)}_l$, the deviations of the approximation \eqref{eq:UstatI} from its mean are of order $O_{\mathbb{P}}(1/\sqrt{B})$, since it is a basic average of $B$ i.i.d. terms.
\end{remark}

From the theorem stated above, one may straightforwardly deduce a bound on the excess risk of kernels $\widehat{H}_{B}$ minimizing the incomplete version of the empirical risk based on $B$ terms, \textit{i.e.} such that
\begin{equation}\label{eq:inc_emp_min}
\widetilde{U}_{B}\left(\widehat{H}_{B}\right)= \min_{H \in \mathcal{H}}\widetilde{U}_{B}(H).
\end{equation}

\begin{corollary}\label{cor} Let $\mathcal{H}$ be a collection of symmetric kernels on $\prod_{k=1}^K\X_{k}^{d_k}$ that satisfies the conditions stipulated in Proposition \ref{prop:Uproc}. Let $\delta>0$. For any minimizer $\widehat{H}_{B}$ of the statistical estimate of the risk \eqref{eq:UstatI}, the following assertions hold true
\begin{itemize}
\item[(i)] We have with probability at least $1-\delta$: $\forall \mathbf{n}\in \mathbb{N}^{*K}$, $\forall B\geq 1$,
\begin{multline*}
\mu (\widehat{H}_{B}) - \inf_{H\in \mathcal{H}}\mu(H) \leq 2\mathcal{M}_{\mathcal{H}}\times\\ \left\{
2\sqrt{\frac{2 V\log(1+N)}{N}}+\sqrt{\frac{\log(2/\delta)}{N}}+\sqrt{2\frac{V\log(1+\#\Lambda)+ \log(4/\delta)}{B}}\right\}.
\end{multline*}
\item[(ii)] We have: $\forall \mathbf{n}\in \mathbb{N}^{*K}$, $\forall B\geq 1$,
\begin{multline*}
\mathbb{E}\left[ \sup_{H\in \mathcal{H}}\left\vert \widetilde{U}_B(H)-\mu(H)  \right\vert  \right]\leq \mathcal{M}_{\mathcal{H}}\left\{ 2 \sqrt{\frac{2V\log (1+N)}{N}}+\sqrt{\frac{2(\log 2 + V\log (1+\#\Lambda))}{B}}\right\}.
\end{multline*}
\end{itemize}
\end{corollary}

The first assertion of Theorem \ref{thm:main} provides a control of the deviations between the $U$-statistic \eqref{eq:UstatG} and its incomplete counterpart \eqref{eq:UstatI} uniformly over the class $\mathcal{H}$. As the number of terms $B$ increases, this deviation decreases at a rate of $O(1/\sqrt{B})$. The second assertion of Theorem \ref{thm:main} gives a maximal deviation result with respect to $\mu(H)$.
Observe in particular that, with the asymptotic settings previously specified, $N = O(n)$ and $\log(\#\Lambda) = O(\log n)$ as $n\rightarrow +\infty$. The bounds stated above thus show that, for a number $B=B_{n}$ of terms tending to infinity at a rate $O(n)$ as $%
n\rightarrow +\infty $, the maximal deviation $\sup_{H\in 
\mathcal{H}}|\widetilde{U}_{B}(H)-\mu (H)|$ is asymptotically of the
order $O_{\mathbb{P}}((\log(n)/n)^{1/2})$, just like $\sup_{H\in \mathcal{H}}|U_{%
\mathbf{n}}(H)-\mu (H)|$, see bound \eqref{eq:unif_Uproc} in Proposition~\ref{prop:Uproc}.
In short, when considering an incomplete $U$-statistic \eqref{eq:UstatI} with $B=O(n)$ terms only, the learning rate for the corresponding minimizer is of the same order as that of the minimizer of the complete risk \eqref{eq:UstatG}, whose computation requires to average $\#\Lambda = O(n^{d_1+\ldots+d_K})$ terms.
Minimizing such incomplete $U$-statistics thus yields a
significant gain in terms of computational cost while fully preserving the learning rate.
In contrast, as implied by Proposition~\ref{prop:Uproc}, the minimization of a complete $U$-statistic involving $O(n)$ terms, obtained by drawing subsamples of sizes $n'_k = O(n^{1/(d_1+\ldots+d_K)})$ uniformly at random, leads to a rate of convergence of $O(\sqrt{\log(n)/n^{1/(d_1+\ldots+d_K)} })$, which is much slower except in the trivial case where $K=1$ and $d_1=1$.
 These striking results are summarized in Table \ref{table:rate_cardinal}.

The important practical consequence of the above is that when $n$ is too large for the complete risk \eqref{eq:UstatG} to be used, one should instead use the incomplete risk \eqref{eq:UstatI} (setting the number of terms $B$ as large as the computational budget allows).

\begin{table}[t]
\centering
\begin{tabularx}{\textwidth}{lll}
\toprule
Empirical risk criterion & Nb of terms & Rate bound \\
\midrule
\vspace*{0.4cm}
Complete $U$-statistic  & $O(n^{d_1+\ldots+d_K})$      & $O_{\mathbb{P}}(\sqrt{\log(n)/n})$     \\ 
\bigskip
Complete $U$-statistic based on subsamples & $O(n)$        & $O_{\mathbb{P}}\left(\sqrt{\log(n)/n^{\frac{1}{d_1+\ldots+d_K}}}\right)$     \\
\bigskip
\textbf{Incomplete $\boldsymbol{U}$-statistic (our result)}  & $\boldsymbol{O(n)}$        & $\boldsymbol{O_{\mathbb{P}}(\sqrt{\log(n)/n})}$    \\ 
\bottomrule 
\end{tabularx}
\caption{Rate bound for the empirical minimizer of several empirical risk criteria \textit{versus} the number of terms involved in the computation of the criterion. For a computational budget of $O(n)$ terms, the rate bound for the incomplete $U$-statistic criterion is of the same order as that of the complete $U$-statistic, which is a huge improvement over a complete $U$-statistic based on a subsample.}
\label{table:rate_cardinal}
\end{table}

\subsection{Model Selection Based on Incomplete $U$-Statistics}
\label{subsec:modelselect}

Automatic selection of the model complexity is a crucial issue in machine learning: it includes the number of clusters in cluster analysis \citep[see][]{CLEM14} or the choice of the number of possible values taken by a piecewise constant scoring function in multipartite ranking for instance  \citep[\cf][]{CV09ieee}. In the present situation, this boils down to choosing the adequate level of complexity of the class of kernels $\mathcal{H}$, measured through its (supposedly finite) {\sc VC} dimension for simplicity, in order to minimize the (theoretical) risk of the empirical minimizer. It is the purpose of this subsection to show that the incomplete $U$-statistic \eqref{eq:UstatI} can be used to define a penalization method to select a prediction rule with nearly minimal risk, avoiding procedures based on data splitting/resampling and extending the celebrated \textit{structural risk minimization} principle, see \cite{Vapnik99}. Let $\mathcal{H}$ be the collection of all symmetric kernels on $\prod_{k=1}^K\X_{k}^{d_k}$ and set $\mu^*=\inf_{H\in \mathcal{H}}\mu(H)$.
Let $\mathcal{H}_{1},\mathcal{H}_{2},\; \ldots$ be  a sequence of uniformly bounded major subclasses of $\mathcal{H}$, of increasing complexity (\textit{\sc VC} dimension). For any $m\geq 1$, let $V_m$ denote the {\sc VC} dimension of the class $\mathcal{H}_m$ and set $\mathcal{M}_{\mathcal{H}_m}=\sup_{(H,x)\in \mathcal{H}_m\times \X}\vert H(x)\vert<+\infty$. We suppose that there exists $\mathcal{M}<+\infty$ such that $\sup_{m\geq 1}\mathcal{M}_{\mathcal{H}_m}\leq \mathcal{M}$. Given $1\leq B\leq \# \Lambda$ and $m\geq 1$, the complexity penalized empirical risk of a solution $\widetilde{U}_{B,m}$ of the ERM problem \eqref{eq:inc_emp_min} with $\mathcal{H}=\mathcal{H}_m$ is
\begin{equation}
\widetilde{U}_B(\widehat{H}_{B,m})+\text{pen}(B,m),
\end{equation}
where the quantity $\text{pen}(B,m)$ is a \textit{distribution free} penalty given by:
\begin{eqnarray}\label{eq:penalty}
\text{pen}(B,m)&=&2\mathcal{M}_{\mathcal{H}_m}\left\{  \sqrt{\frac{2V_m\log (1+N)}{N}}+\sqrt{\frac{2(\log 2 + V_m\log (1+\#\Lambda))}{B}}\right\}\nonumber\\
&+& 2\mathcal{M}\sqrt{\frac{(B+n)\log m}{B^2}}.
\end{eqnarray}
As shown in Assertion $(ii)$ of Corollary \ref{cor}, the quantity above is an upper bound for the expected maximal deviation $\mathbb{E}[\sup_{H\in \mathcal{H}_m}\vert \widetilde{U}_B(H) -\mu(H) \vert ]$ and is thus a natural penalty candidate to compensate the overfitting within class $\mathcal{H}_m$. We thus propose to select
\begin{equation}
\widehat{m}_B=\argmin_{m\geq 1}\left\{ \widetilde{U}_B(\widehat{H}_{B,m})+\text{pen}(B,m)  \right\}.
\end{equation}
As revealed by the theorem below, choosing $B=O(n)$, the prediction rule $\widehat{H}_{\widehat{m}_B}$ based on a penalized criterion involving the summation of $O(n)$ terms solely, achieves a nearly optimal trade-off between the bias and the distribution free upper bound \eqref{eq:penalty} on the variance term.

\begin{theorem}\label{thm:selec} {\sc (Oracle inequality)}
Suppose that Theorem \ref{thm:main}'s assumptions are fulfilled for all $m\geq 1$ and that $\sup_{m\geq 1}\mathcal{M}_{\mathcal{H}_m}\leq \mathcal{M}<+\infty$. Then, we have: $\forall \mathbf{n}\in \mathbb{N}^{*K}$, $\forall B\in\{1,\; \ldots,\; \#\Lambda  \}$, 
\begin{equation*}
\mu (\widehat{H}_{B,\widehat{m}}) - \mu^{*} \leq \inf_{k \geq 1} \left\{\inf_{H\in \mathcal{H}_m}\mu(H)-\mu^{*}+ \textnormal{pen}(B,m)\right\}+\mathcal{M}\frac{\sqrt{2\pi(B+n)}}{B}.
\end{equation*}
\end{theorem}

We point out that the argument used to obtain the above result can be straightforwardly extended to other (possibly data-dependent) complexity penalties \citep[\cf][]{Mas06}, see the proof in Appendix~\ref{app:thmselec}.

\subsection{Fast Rates for ERM of Incomplete $U$-Statistics}
\label{subsec:fastrates}

In \cite{CLV08}, it has been proved that, under certain ``low-noise'' conditions, the minimum variance property of the $U$-statistics used to estimate the ranking risk (corresponding to the situation $K=1$ and $d_1=2$) leads to learning rates faster than $O_{\mathbb{P}}(1/\sqrt{n})$. These results rely on the \textit{Hajek projection}, a linearization technique originally introduced in \cite{Hoeffding48} for the case of one sample $U$-statistics and next extended to the analysis of a much larger class of functionals in \cite{Haj68}. It consists in writing $U_{\mathbf{n}}(H)$ as the sum of the orthogonal projection
\begin{equation}
\label{eq:Hajek}
\widehat{U}_{\mathbf{n}}(H)=\sum_{k=1}^K\sum_{i=1}^{n_k}\mathbb{E}\left[U_{\mathbf{n}}(H)\mid X_i^{(k)} \right]-(n-1)\mu(H),
\end{equation}
which is itself a sum of $K$ independent basic sample means based on i.i.d. r.v.'s (of the order $O_{\mathbb{P}}(1/%
\sqrt{n})$ each, after recentering), plus a possible negligible term. This representation was used for instance by \cite{GramsSerfling73} to refine the CLT in the multisample $U$-statistics framework. Although useful as a theoretical tool, it should be noticed that the quantity $\widehat{U}_{\mathbf{n}}(H)$ is not of practical interest, since the conditional expectations involved in the summation are generally unknown.

Although incomplete $U$-statistics do not share the minimum variance property (see Section~\ref{subsec:approx}), we will show that the same fast rate bounds for the excess risk as those reached by ERM of $U$-statistics (corresponding to the summation of $O(n^2)$ pairs of observations) can be attained by empirical ranking risk minimizers, when estimating the ranking risk by incomplete $U$-statistics involving the summation of $o(n^2)$ terms solely. 

\par For clarity (and comparison purpose), we first recall the statistical learning framework considered in \cite{CLV08}. Let $(X,Y)$ be a pair of random variables defined on the same probability space, where $Y$ is a real-valued label and $X$ models some input information taking its values in a measurable space $\mathcal{X}$ hopefully useful to predict $Y$. Denoting by $(X',Y')$ an independent copy of the pair $(X,Y)$. The goal pursued here is to learn how to rank the input observations $X$ and $X'$, by means of an antisymmetric \textit{ranking rule} $r:\mathcal{X}^2\rightarrow \{ -1,\; +1 \}$ (\textit{i.e.} $r(x,x')=-r(x'x)$ for any $(x,x')\in\mathcal{X}^2$), so as to minimize the \textit{ranking risk}
\begin{equation}\label{eq:ranking_risk}
L(r)=\mathbb{P}\{(Y-Y')\cdot r(X,X')<0  \}.
\end{equation}
The minimizer of the ranking risk is the ranking rule $r^*(X,X')=2\mathbb{I}\{\mathbb{P}\{Y>Y '\mid (X,X')\} \geq \mathbb{P}\{Y<Y '\mid (X,X') \}-1$ \citep[see Proposition 1 in][]{CLV08}.
The natural empirical counterpart of \eqref{eq:ranking_risk} based on a sample of independent copies $(X_1,Y_1),\; \ldots,\; (X_n,Y_n)$ of the pair $(X,Y)$ is the $1$-sample $U$-statistic $U_n(H_r)$ of degree two with kernel $H_r((x,y),(x',y'))=\mathbb{I}\{ (y-y')\cdot r(x,x')<0 \}$ for all $(x,y)$ and $(x',y'))$ in $\mathcal{X}\times \mathbb{R}$ given by:
\begin{equation}\label{eq:emp_ranking_risk}
L_n(r)=U_n(H_r)=\frac{2}{n(n-1)}\sum_{i<j}\mathbb{I}\{(Y_i-Y_j)\cdot r(X_i,X_j)<0  \}.
\end{equation}
Equipped with these notations, a statistical version of the excess risk $\Lambda(r)=L(r)-L(r^*)$ is a $U$-statistic $\lambda_n(r)$ with kernel $q_r=H_r-H_{r^*}$. The key ``noise-condition'', which allows to exploit the Hoeffding/Hajek decomposition of $\Lambda_n(r)$, is stated below.
\begin{assumption}\label{assump:noise}
There exist constants $c>0$ and $\alpha\in[0,1]$ such that:
$$
\forall r\in \mathcal{R},\;\; Var\left(h_r(X,Y)\right )\leq c \Lambda(r)^{\alpha},
$$
where we set $h_r(x,y)=\mathbb{E}[q_r((x,y),(X',Y')]$.
\end{assumption}
Recall incidentally that very general sufficient conditions guaranteeing that this assumption holds true have been exhibited, see Section 5 in \cite{CLV08} (notice that the condition is void for $\alpha=0$). Since our goal is to explain the main ideas rather than achieving a high level of generality, we consider a very simple setting, stipulating that the cardinality of the class of ranking rule candidates $\mathcal{R}$ under study is finite, $\#\mathcal{R}=M<+\infty$, and that the optimal rule $r^*$ belongs to $\mathcal{R}$. The following proposition is a simplified version of the fast rate result proved in \cite{CLV08} for the empirical minimizer $\widehat{r}_n=\argmin_{r\in \mathcal{R}}L_n(r)$.
\begin{proposition}(\cite{CLV08}, \textsc{Corollary 6}) Suppose that Assumption \ref{assump:noise} is fulfilled.
Then, there exists a universal constant $C>0$ such that for all $\delta\in (0,1)$, we have: $\forall n\geq 2$,
\begin{equation}
L(\widehat{r}_n)-L(r^*)\leq C\left( \frac{\log(M/\delta)}{n} \right)^{\frac{1}{2-\alpha}}.
\end{equation}
\end{proposition}
Consider now the minimizer $\widetilde{r}_B$ of the incomplete $U$-statistic risk estimate
\begin{equation}
\widetilde{U}_B(H_r)=\frac{1}{B}\sum_{k=1}^B\sum_{(i,j):1\leq i<j\leq n}\epsilon_k((i,j))\mathbb{I}\{(Y_i-Y_j)\cdot r(X_i,X_j)<0  \}
\end{equation}
 over $\mathcal{R}$, where $\epsilon_k((i,j))$ indicates whether the pair $(i,j)$ has been picked at the $k$-th draw ($\epsilon_k((i,j))=1$ in this case, which occurs with probability $1/\binom{n}{2}$) or not (then, we set $\epsilon_k((i,j))=0$). Observe that $\widetilde{r}_B$ also minimizes the empirical estimate of the excess risk $\widetilde{\Lambda}_B(r)=\widetilde{U}_B(q_r)$ over $\mathcal{R}$.

\begin{theorem}\label{thm:fast}
Let $\alpha\in [0,1]$ and suppose that Assumption~\ref{assump:noise} is fulfilled. If we set $B = O(n^{2/(2-\alpha)})$, there exists some constant $C<+\infty$ such that, for all $\delta\in (0,1)$, we have with probability at least $1-\delta$: $\forall n\geq 2$,
\begin{equation*}
L(\widetilde{r}_B)-L(r^*)\leq C\left(\frac{\log(M/ \delta)}{n}  \right)^{\frac{1}{2-\alpha}}.  
\end{equation*}
\end{theorem}
As soon as $\alpha<1$, this result shows that the same fast rate of convergence as that reached by $\widehat{r}_n$ can be attained by the ranking rule $\widetilde{r}_B$,  which minimizes an empirical version of the ranking risk involving the summation of $O(n^{2/(2-\alpha)})$ terms solely. For comparison purpose, minimization of the criterion \eqref{eq:ranking_risk} computed with a number of terms of the same order leads to a rate bound of order $O_{\mathbb{P}}(n^{1/(2-\alpha)^2})$.

Finally, we point out that fast rates for the clustering problem have been also investigated in \cite{CLEM14}, see Section~5.2 therein. The present analysis can be extended to the clustering framework by means of the same arguments.





\subsection{Alternative Sampling Schemes}
\label{subsec:sampling}

\begin{figure}[t]
\centering
\includegraphics[width=0.6\textwidth]{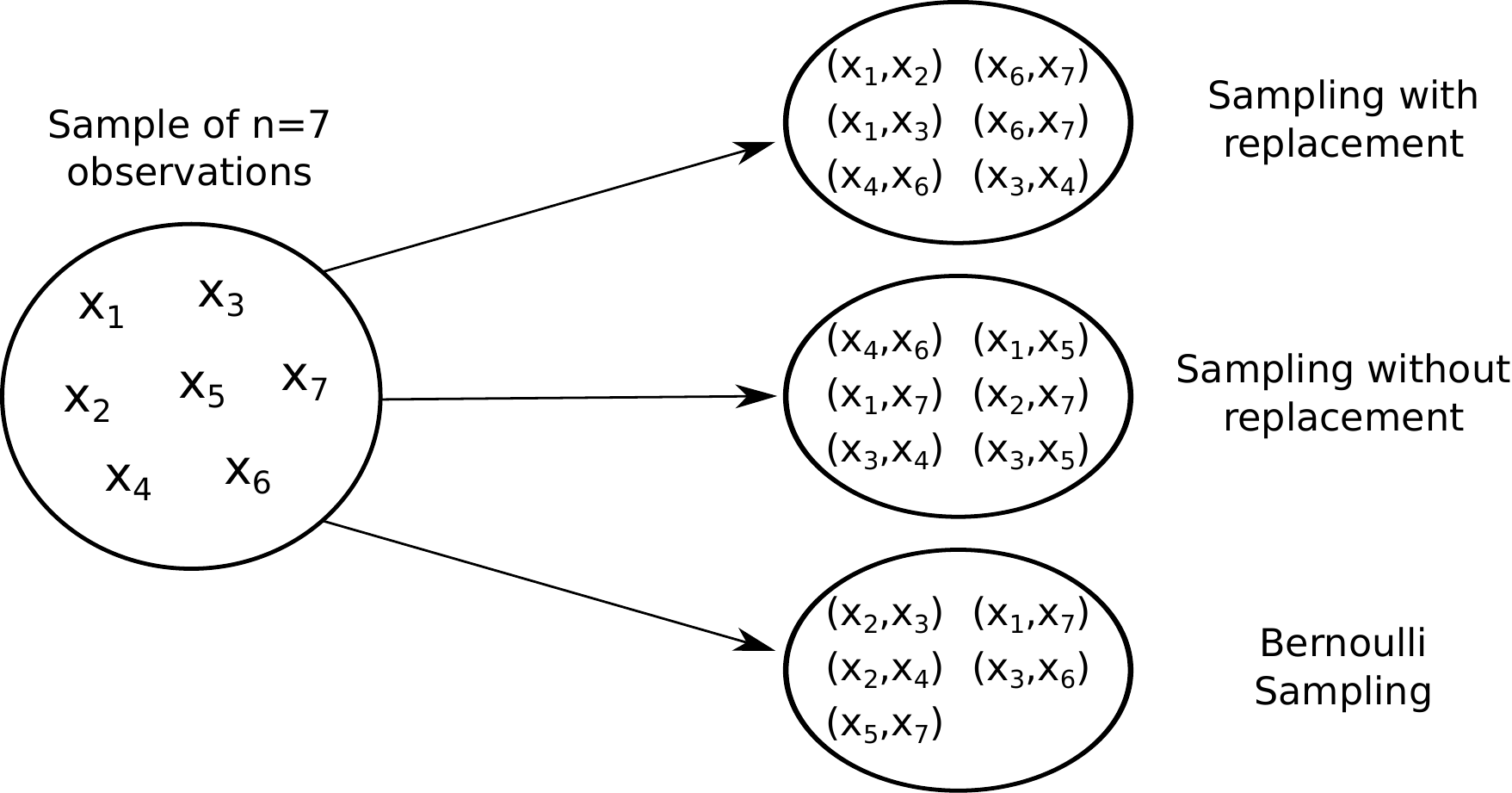}
\caption{Illustration of different sampling schemes for approximating a $U$-statistic. For simplicity, consider again the case $K=1$ and $d_1=2$. Here $n=7$ and the expected number of terms is $B=6$. Sampling with or without replacement results in exactly $B$ terms, with possible repetitions when sampling with replacement, \emph{e.g.} $(x_6,x_7)$ in this example. In contrast, Bernoulli sampling with $\pi_I=B/\#\Lambda$ results in $B$ terms only in expectation, with individual realizations that may exhibit more or fewer terms.}
\label{fig:samp_illustr}
\end{figure}

Sampling with replacement is not the sole way of approximating generalized $U$-statistics with a controlled computational cost. As proposed in \cite{Janson84}, other sampling schemes can be considered, Bernoulli sampling or sampling without replacement in particular (see Figure~\ref{fig:samp_illustr} for an illustration). We now explain how the results of this paper can be extended to these situations. The population of interest is the set $\Lambda$ and a \textit{survey sample} of (possibly random) size $b \leq n$ is any subset $s$ of cardinality $b =b(s)$ less than $\#\Lambda$ in the power set $\mathcal{P}(\Lambda)$. Here, a general  \textit{survey scheme without replacement} is any conditional probability distribution $R$ on the set of all possible samples $s\in \mathcal{P}(\Lambda)$ given $(\mathbf{X}_{I})_{I\in \Lambda}$. For any $I\in \Lambda$, the first order \textit{inclusion probability} $\pi_I(R)= \mathbb{P}_{R}\{I \in S\}$,
 is the probability that the unit $I$ belongs to a random sample $S$ drawn from distribution $R$. We set $\boldsymbol{\pi}(R)= (\pi_I(R))_{I\in \Lambda}$. 
 The second order inclusion probabilities are denoted by
 $\pi_{I,J}(R) = \mathbb{P}_{R}\{(I,J)\in S^2\}$ for any $I\neq J$ in $\Lambda$. When no confusion is possible, we omit to mention the dependence in $R$ when writing the first/second order probabilities of inclusion. The information related to the observed sample $S \subset \Lambda$ is fully enclosed in the random vector $\boldsymbol{\Delta} = (\Delta(I))_{I\in \Lambda}$, where $\Delta(I) =\mathbb{I}\{ I\in S \}$ for all $I\in \Lambda$.
 The $1$-d marginal distributions of the sampling scheme $\boldsymbol{\Delta}_n$ are the Bernoulli distributions with parameters $\pi_{I}$, $I\in \Lambda$, and the covariance matrix of the r.v. $\boldsymbol{\Delta}_{\mathbf{n}}$ is given by
 $ \Gamma = \left\{\pi_{I,J} - \pi_{I}\pi_{J}  \right\}_{I, J}$ with the convention $\pi_{I,I}=\pi_I$ for all $I\in \Lambda$.
 Observe that, equipped with the notations above, $\sum_{I\in \Lambda}\Delta(I) = b(S)$.
 
 \par One of the simplest survey plans is the Poisson scheme (without replacement), for which the $\Delta(I)$'s are independent Bernoulli random variables with parameters $\pi_I$, $I\in \Lambda$, in $(0,1)$. The first order inclusion probabilities fully characterize such a plan. Observe in addition that the size $b(\mathcal{S})$ of a sample generated this way is random with expectation $B=\mathbb{E}[b(S)\mid (\mathbf{X}_{I})_{I\in \Lambda}]=\sum_{I\in \Lambda}\pi_I$. The situation where the $\pi_I$'s are all equal corresponds to the Bernoulli sampling scheme: $\forall I\in \Lambda$, $\pi_I=B/\#\Lambda$. The Poisson survey scheme plays a crucial role in sampling theory, inso far as a wide range of survey schemes can be viewed as conditional Poisson schemes, see \cite{Hajek64}. 
 For instance, one may refer to \cite{Cochran77} or \cite{Dev87} for accounts of survey sampling techniques. 
 
 Following in the footsteps of the seminal contribution of \cite{HT51}, an estimate of \eqref{eq:UstatG} based on a sample drawn from a survey scheme $R$ with first order inclusion probabilities $(\pi_I)_{I\in \Lambda}$ is given by:
 \begin{equation}\label{eq:HT}
 \bar{U}_{HT}(H)=\frac{1}{\#\Lambda}\sum_{I\in \Lambda}\frac{\Delta(I)}{\pi_I}H(\mathbf{X}_I),
 \end{equation}
 with the convention that $0/0=0$. Notice that it is an unbiased estimate of \eqref{eq:UstatG}: 
 $$
 \mathbb{E}[ \bar{U}_{HT}(H) \mid (\mathbf{X}_{I})_{I\in \Lambda} ]=U_{\mathbf{n}}(H).$$
 In the case where the sample size is deterministic, its conditional variance is given by:
 $$
 Var( \bar{U}_{HT}(H) \mid (\mathbf{X}_{I})_{I\in \Lambda})=\frac{1}{2}\sum_{I\neq J}\left(\frac{H(\mathbf{X}_I)}{\pi_I}-\frac{H(\mathbf{X}_J)}{\pi_J}  \right)^2(\pi_{I,J}-\pi_I\pi_J).
 $$
 We point out that the computation of \eqref{eq:HT} involves summing over a possibly random number of terms, equal to $B=\mathbb{E}[b(S)]=\sum_{I\in \Lambda}\pi_I$ in average and whose variance is equal to $Var(b(S))=\sum_{I\in \Lambda}\pi_I(1-\pi_I)+\sum_{I\neq J}\{\pi_{I,J}-\pi_I\pi_J\}$.
 
 Here, we are interested in the situation where the $\Delta(I)$'s are independent from $(X_I)_{I\in \Lambda}$, and either a sample of size $B\leq \#\Lambda$ fixed in advance is chosen uniformly at random among the $\binom{\#\Lambda}{B}$ possible choices (this survey scheme is sometimes referred to as \textit{rejective sampling} with equal first order inclusion probabilities), or else it is picked by means of a Bernoulli sampling with parameter $B/\#\Lambda$. Observe that, in both cases, we have $\pi_I=B/\#\Lambda$ for all $I\in \Lambda$. The following theorem shows that in both cases, similar results as those obtained for \textit{sampling with replacement} can be derived for minimizers of the Horvitz-Thompson risk estimate \eqref{eq:HT}.
 
 \begin{theorem}\label{thm:main2}  Let $\mathcal{H}$ be a collection of bounded symmetric kernels on $\prod_{k=1}^K\mathcal{X}_{k}^{d_k}$ that fulfills the assumptions involved in Proposition \ref{prop:Uproc}. Let $B\in \{1,\; \ldots,\; \# \Lambda  \}$. Suppose that, for any $H\in \mathcal{H}$, $\bar{U}_{HT}(H)$ is the incomplete $U$-statistic based on either a Bernoulli sampling scheme with parameter $B/\#\Lambda$ or else a sampling without replacement scheme of size $B$.
 For all $\delta\in (0,1)$, we have with probability at least $1-\delta$: $\forall \mathbf{n}\in \mathbb{N}^{*K}$, $\forall B\in \{1,\; \ldots,\; \# \Lambda  \}$,
 \begin{equation*}
 \sup_{H\in \mathcal{H}}\left\vert \bar{U}_{HT}(H)- U_{\mathbf{n}}(H)\right\vert \leq 2\mathcal{M}_{\mathcal{H}}\sqrt{\frac{\log(2(1+\#\Lambda)^V/\delta)}{B}}+\frac{2\log(2(1+\#\Lambda)^V/\delta)\mathcal{M}_{\mathcal{H}}}{3B},
 \end{equation*}
 in the case of the Bernoulli sampling design, and
  \begin{equation*}
  \sup_{H\in \mathcal{H}}\left\vert \bar{U}_{HT}(H)- U_{\mathbf{n}}(H)\right\vert \leq \sqrt{2}\mathcal{M}_{\mathcal{H}} \sqrt{\frac{\log(2(1+\#\Lambda)^V/\delta)}{B} },
  \end{equation*}
  in the case of the sampling without replacement design.
 \end{theorem}

We highlight the fact that, from a computational perspective, sampling with replacement is undoubtedly much more advantageous than Bernoulli sampling or sampling without replacement. Indeed, although its expected value is equal to $B$, the size of a Bernoulli sample is stochastic and the related sampling algorithm requires a loop through the elements $I$ of $\Lambda$ and the practical implementation of sampling without replacement is generally based on multiple iterations of sampling with replacement, see \cite{tille2006sampling}.

\section{Application to Stochastic Gradient Descent for ERM}\label{sec:SGD}

The theoretical analysis carried out in the preceding sections focused on the properties of empirical risk minimizers but ignored the issue of finding such a minimizer. In this section, we show that the sampling technique introduced in Section \ref{sec:approx} also provides practical means of scaling up iterative statistical learning techniques. Indeed, large-scale training of many machine learning models, such as {\sc SVM}, {\sc Deep Neural Networks} or {\sc soft $K$-means} among others, is based on stochastic gradient descent (SGD in abbreviated form), see \cite{B98}. When the risk is of the form \eqref{eq:parameter},  we now investigate the benefit of using, at each iterative step, a gradient estimate of the form of an incomplete $U$-statistic, instead of an estimate of the form of a complete $U$-statistic with exactly the same number of terms based on subsamples drawn uniformly at random.

Let $\Theta\subset \mathbb{R}^q$ with $q\geq 1$ be some parameter space and $H:\prod_{k=1}^K\mathcal{X}_k^{d_k}\times \Theta\rightarrow \mathbb{R}$ be a loss function which is convex and differentiable in its last argument. Let $(X^{(k)}_1,\; \ldots,\; X^{(k)}_{d_k})$, $1\leq k\leq K$, be $K$ independent random vectors with distribution $F_k^{\otimes d_k}(dx)$ on $\X_k^{d_k}$ respectively such that the random vector $H(X^{(1)}_1,\; \ldots,\; X^{(1)}_{d_1},\; \ldots,\; X^{(K)}_1,\; \ldots,\; X^{(K)}_{d_K};\; \theta)$ is square integrable for any $\theta\in \Theta$. For all $\theta\in \Theta$, set 
$$
L(\theta)=\mathbb{E}[H(X^{(1)}_1,\; \ldots,\; X^{(1)}_{d_1},\; \ldots,\; X^{(K)}_1,\; \ldots,\; X^{(K)}_{d_K};\; \theta)]=\mu(H(\cdot ;\; \theta))
$$ and consider the \textit{risk minimization} problem $\min_{\theta\in \Theta}L( \theta)$.
Based on $K$ independent i.i.d. samples $X^{(k)}_1,\; \ldots,\; X^{(k)}_{n_k}$ with $1\leq k\leq K$, the empirical version of the risk function is
$\theta\in \Theta\mapsto \widehat{L}_{\mathbf{n}}(\theta)\overset{def}{=}U_{\mathbf{n}}(H(\cdot;\; \theta))$. Here and throughout, we denote by $\nabla_{\theta}$ the gradient operator w.r.t. $\theta$.

\paragraph{Gradient descent} Many practical machine learning algorithms use variants of the standard gradient descent method, following the iterations:
\begin{equation}\label{eq:iter1}
\theta_{t+1}=\theta_t-\eta_t\nabla_{\theta}\widehat{L}_{\mathbf{n}}(\theta_t),
\end{equation}
with an arbitrary initial value $\theta_0\in\Theta$ and a learning rate (step size) $\eta_t\geq 0$ such that $\sum_{t=1}^{+\infty}\eta_t=+\infty$ and $\sum_{t=1}^{+\infty}\eta^2_t<+\infty$.

Here we place ourselves in a large-scale setting, where the sample sizes $n_1,\; \ldots,\; n_K$ of the training data sets are so large that computing the gradient of $\widehat{L}_{\mathbf{n}}$
\begin{equation}\label{eq:emp_gradient}
\widehat{g}_{\mathbf{n}}(\theta)=\frac{1}{\prod_{k=1}^K \binom{n_k}{d_k}}\sum_{I_1}\ldots\sum_{I_K} \nabla_{\theta} H(\mathbf{X}^{(1)}_{I_1} ; \mathbf{X}^{(2)}_{I_2}; \ldots ;\mathbf{X}^{(K)}_{I_K};\; \theta)
\end{equation}
at each iteration \eqref{eq:iter1} is computationally too expensive. Instead, Stochastic Gradient Descent uses an unbiased estimate $\widetilde{g}(\theta)$ of the gradient \eqref{eq:emp_gradient} that is cheap to compute.
A natural approach consists in replacing \eqref{eq:emp_gradient} by a complete $U$-statistic constructed from subsamples of reduced sizes $n'_k<<n_k$ drawn uniformly at random, leading to the following gradient estimate:
\begin{equation}\label{eq:emp_compgradient}
\widetilde{g}_{\mathbf{n}'}(\theta)=\frac{1}{\prod_{k=1}^K \binom{n'_k}{d_k}}\sum_{I_1}\ldots\sum_{I_K} \nabla_{\theta} H(\mathbf{X}^{(1)}_{I_1} ; \mathbf{X}^{(2)}_{I_2}; \ldots ;\mathbf{X}^{(K)}_{I_K};\; \theta),
\end{equation}
where the symbol $\sum_{I_k}$ refers to summation over all $\binom{n'_k}{d_k}$ subsets $\mathbf{X}^{(k)}_{I_k}=( X^{(k)}_{i_1},\;\ldots,\; X^{(k)}_{i_{d_k}})$ related to a set $I_k$ of $d_k$ indexes $1\leq i_1< \ldots <i_{d_k}\leq n'_k$ and $\mathbf{n'}=(n'_1,\; \ldots,\; n'_K)$.

We propose an alternative strategy based on the sampling scheme described in Section~\ref{sec:approx}, \emph{i.e.} a gradient estimate in the form of an \emph{incomplete} $U$-statistic:
\begin{equation}\label{eq:gradient_UstatI}
\widetilde{g}_B(\theta)=\frac{1}{B}\sum_{(I_1,\;\ldots,\, I_K)\in\mathcal{D}_B} \nabla_{\theta}H(\mathbf{X}^{(1)}_{I_1},\;\ldots,\; \mathbf{X}^{(K)}_{I_K};\; \theta),
\end{equation}
where $\D_B$ is built by sampling with replacement in the set $\Lambda$.


It is well-known that the variance of the gradient estimate negatively impacts on the convergence of SGD. Consider for instance the case where the loss function $H$ is $(1/\gamma)$-smooth in its last argument, \ie $\forall \theta_1,\theta_2\in\Theta$:
$$\|\nabla_{\theta}H(\cdot;\;\theta_1) - \nabla_{\theta}H(\cdot;\;\theta_2)\| \leq \frac{1}{\gamma}\|\theta_1-\theta_2\|.$$
Then one can show that if $\widetilde{g}$ is the gradient estimate:
\begin{eqnarray*}
\mathbb{E}[\widehat{L}_{\mathbf{n}}(\theta_{t+1})] &=& \mathbb{E}[\widehat{L}_{\mathbf{n}}(\theta_t-\eta_t\widetilde{g}(\theta_t))]\\
& \leq & \mathbb{E}[\widehat{L}_{\mathbf{n}}(\theta_{t})] - \eta_t\|\mathbb{E}[\widehat{g}_{\mathbf{n}}(\theta_{t})]\|^2 + \frac{\eta_t^2}{2\gamma}\mathbb{E}[\|\widetilde{g}(\theta_t)\|^2]\\
& \leq & \mathbb{E}[\widehat{L}_{\mathbf{n}}(\theta_{t})] - \eta_t\left(1-\frac{\eta_t}{2\gamma}\right)\mathbb{E}[\|\widehat{g}_{\mathbf{n}}(\theta_{t})\|^2] + \frac{\eta_t^2}{2\gamma}\var[\widetilde{g}(\theta_t)].
\end{eqnarray*}
In other words, the smaller the variance of the gradient estimate, the larger the expected reduction in objective value. Some recent work has focused on variance-reduction strategies for SGD when the risk estimates are basic sample means \citep[see for instance][]{Le-Roux2012a,Johnson2013a}.

In our setting where the risk estimates are of the form of a $U$-statistic, we are interested in comparing the variance of $\widetilde{g}_{\mathbf{n}'}(\theta)$ and $\widetilde{g}_B(\theta)$ when $B = \prod_{k=1}^K \binom{n'_k}{d_k}$ so that their computation requires to average over the same number of terms and thus have similar computational cost.\footnote{Note that sampling $B$ sets from $\Lambda$ to obtain \eqref{eq:gradient_UstatI} is potentially more efficient than sampling $n'_k$ points from $\mathbf{X}_{\{1,...,n_k\}}$ for each $k=1,\dots,K$ and then forming all combinations to obtain \eqref{eq:emp_compgradient}.} Our result is summarized in the following proposition.

\begin{proposition}
\label{prop:var_comp}
Let $B = \prod_{k=1}^K \binom{n'_k}{d_k}$ for $n'_k\ll n_k$, $k=1,\; \ldots,\; K$. In the asymptotic framework \eqref{asymptotics}, we have:
\begin{equation*}
\var[\widetilde{g}_{\mathbf{n}'}(\theta)] = O\left(\frac{1}{\sum_{k=1}^K n'_k}\right),\quad
\var[\widetilde{g}_B(\theta)] = O\left(\frac{1}{\prod_{k=1}^K \binom{n'_k}{d_k}}\right),
\end{equation*}
as $n'=n'_1+\ldots+n'_K\rightarrow +\infty$.
\end{proposition}

Proposition~\ref{prop:var_comp} shows that the convergence rate of $\var[\widetilde{g}_B(\theta)]$ is faster than that of $\var[\widetilde{g}_{\mathbf{n}'}(\theta)]$ except when $K=1$ and $d_1=1$. Thus the expected improvement in objective function at each SGD step is larger when using a gradient estimate in the form of \eqref{eq:gradient_UstatI} instead of \eqref{eq:emp_compgradient}, although both strategies require to average over the same number of terms. This is also supported by the experimental results reported in the next section.



\section{Numerical Experiments}\label{sec:exp}

We show the benefits of the sampling approach promoted in this paper on two applications: metric learning for classification, and model selection in clustering.

\subsection{Metric Learning}

In this section, we focus on the metric learning problem (see Section~\ref{sec:metriclearning}). As done in much of the metric learning literature, we restrict our attention to the family of pseudo-distance functions $D_{\boldsymbol{M}}:\mathbb{R}^d\times\mathbb{R}^d\rightarrow\mathbb{R}_+$ defined as
$$D_{\boldsymbol{M}}(\boldsymbol{x},\boldsymbol{x}') = (\boldsymbol{x}-\boldsymbol{x}')\boldsymbol{M}(\boldsymbol{x}-\boldsymbol{x}')^T,$$
where $\boldsymbol{M}\in\mathbb{S}_+^{d}$, and $\mathbb{S}_+^{d}$ is the cone of $d\times d$ symmetric positive-semidefinite (PSD) matrices.

Given a training sample $\{(\boldsymbol{x}_i,y_i)\}_{i=1}^n$ where $\boldsymbol{x}_i\in\mathbb{R}^d$ and $y_i\in\{1,\dots,C\}$, let $y_{ij}=1$ if $y_i=y_j$ and 0 otherwise for any pair of samples. Given a threshold $b\geq 0$, we define the empirical risk as follows:
\begin{equation}
\label{eq:riskexp}
R_{n}(D_{\boldsymbol{M}})= \frac{2}{n(n-1)}\sum_{1\leq i<j\leq n}\left[y_{ij}(b-D_{\boldsymbol{M}}(\boldsymbol{x}_{i},\boldsymbol{x}_{j}))\right]_+,
\end{equation}
where $[u]_+ = \max(0,1-u)$ is the hinge loss. This risk estimate is convex and was used for instance by \citet{Jin2009a} and \citet{Cao2012a}. Our goal is the find the empirical risk minimizer among our family of distance functions, i.e.:
\begin{equation}
\label{eq:ermexp}
\widehat{\boldsymbol{M}} = \argmin_{\boldsymbol{M}\in\mathbb{S}_+^{d}} R_{n}(D_{\boldsymbol{M}}).
\end{equation}
In our experiments, we use the following two data sets:
\begin{itemize}
\item \textbf{Synthetic data set}: some synthetic data that we generated for illustration. $X$ is a mixture of 10 gaussians in $\mathbb{R}^{40}$ -- each one corresponding to a class -- such that all gaussian means are contained in an subspace of dimension 15 and their shared covariance matrix is proportional to identity with a variance factor such that some overlap is observed. That is, the solution to the metric learning problem should be proportional to the linear projection over the subspace containing the gaussians means. Training and testing sets contain respectively 50,000 and 10,000 observations.
\item \textbf{MNIST data set}: a handwritten digit classification data set which has 10 classes and consists of 60,000 training images and 10,000 test images.\footnote{\url{http://yann.lecun.com/exdb/mnist/}} This data set has been used extensively to benchmark metric learning \citep{Weinberger2009a}. As done by previous authors, we reduce the dimension from 784 to 164 using PCA so as to retain 95\% of the variance, and normalize each sample to unit norm.
\end{itemize}
Note that for both datasets, merely computing the empirical risk \eqref{eq:riskexp} for a given $\boldsymbol{M}$ involves averaging over more than $10^9$ pairs. 

We conduct two types of experiment. In Section~\ref{sec:expone}, we subsample the data before learning and evaluate the performance of the ERM on the subsample. In Section~\ref{sec:expsgd}, we use Stochastic Gradient Descent to find the ERM on the original sample, using subsamples at each iteration to estimate the gradient.

\subsubsection{One-Time Sampling}
\label{sec:expone}

We compare two sampling schemes to approximate the empirical risk:
\begin{itemize}
\item Complete $U$-statistic: $p$ indices are uniformly picked at random in $\{1, \ldots, n\}$. The empirical risk is approximated using any possible pair formed by the $p$ indices, that is $\frac{p(p-1)}{2}$ pairs.
\item Incomplete $U$-statistic: the empirical risk is approximated using $\frac{p(p-1)}{2}$ pairs picked uniformly at random in $\{1, \ldots, n\}^2$.
\end{itemize}

For each strategy, we use a projected gradient descent method in order to solve \eqref{eq:ermexp}, using several values of $p$ and averaging the results over 50 random trials. As the testing sets are large, we evaluate the test risk on 100,000 randomly picked pairs.

\begin{figure}[t]
\centering
\subfigure[Synthetic data set]{\includegraphics[width=0.48\columnwidth]{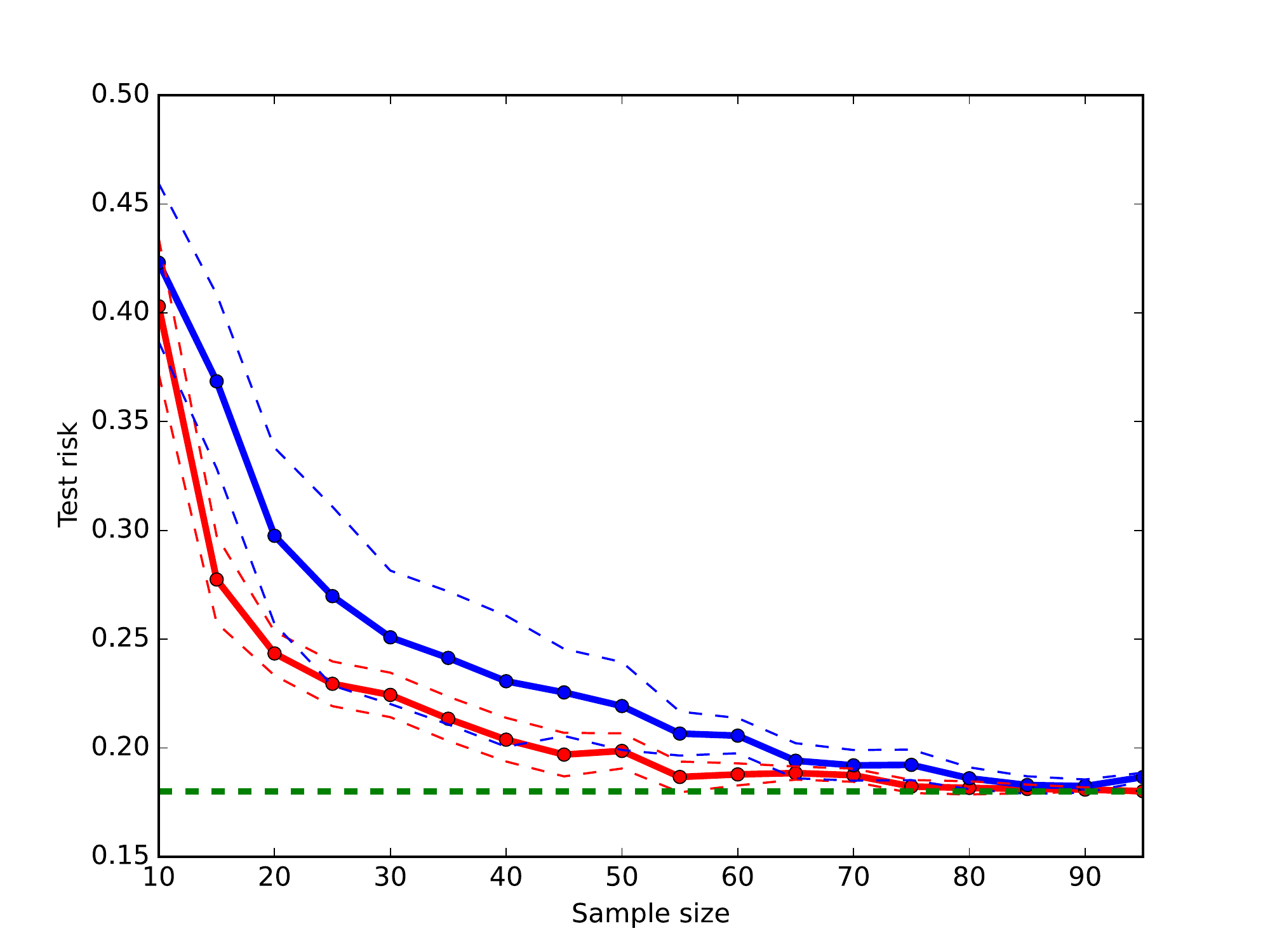}
\label{fig:metric_learning_synthetic}}
\subfigure[MNIST data set]{\includegraphics[width=0.48\columnwidth]{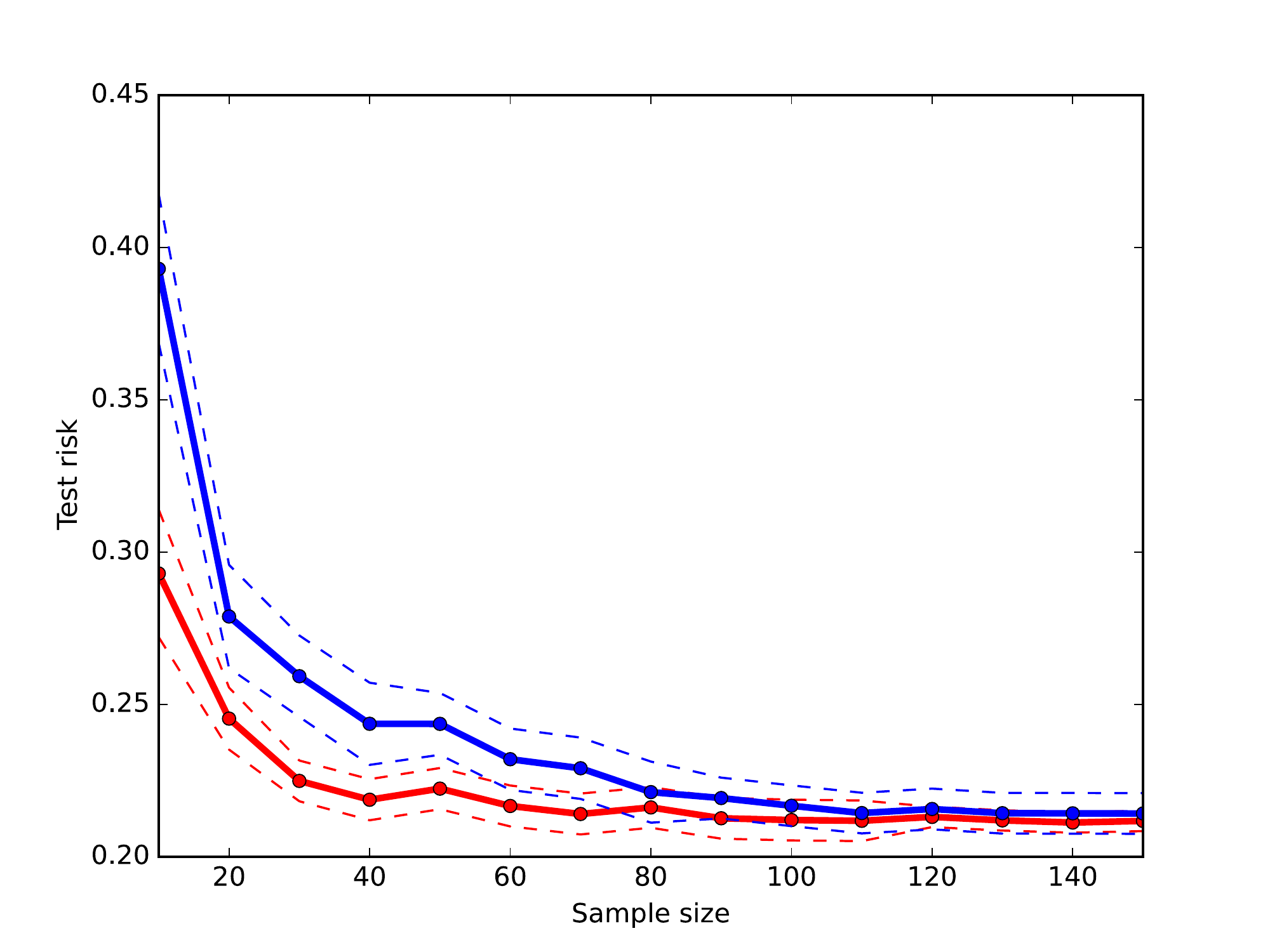} 
\label{fig:metric_learning_mnist}}
 \caption{Test risk with respect to the sample size $p$ of the ERM when the risk is approximated using complete (blue) or incomplete (red) $U$-statistics. Solid lines represent means and dashed ones represent standard deviation. For the synthetic data set, the green dotted line represent the performance of the true risk minimizer.}
 \label{fig:metric_learning1}
\end{figure}

\begin{figure}[t]
\centering
\subfigure[Synthetic data set]{\includegraphics[width=0.48\columnwidth]{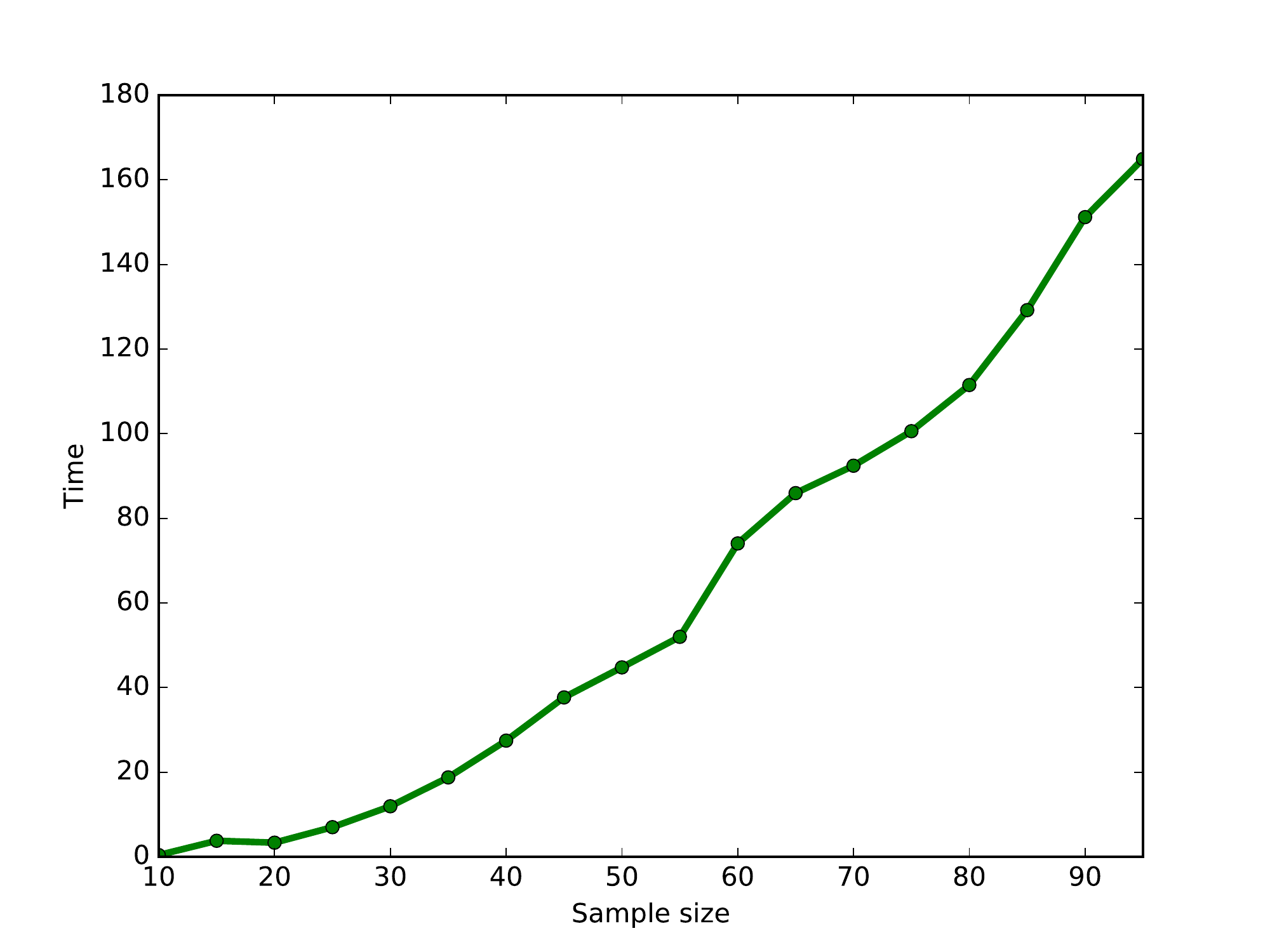}
\label{fig:metric_learning_synthetic_time}}
\subfigure[MNIST data set]{\includegraphics[width=0.48\columnwidth]{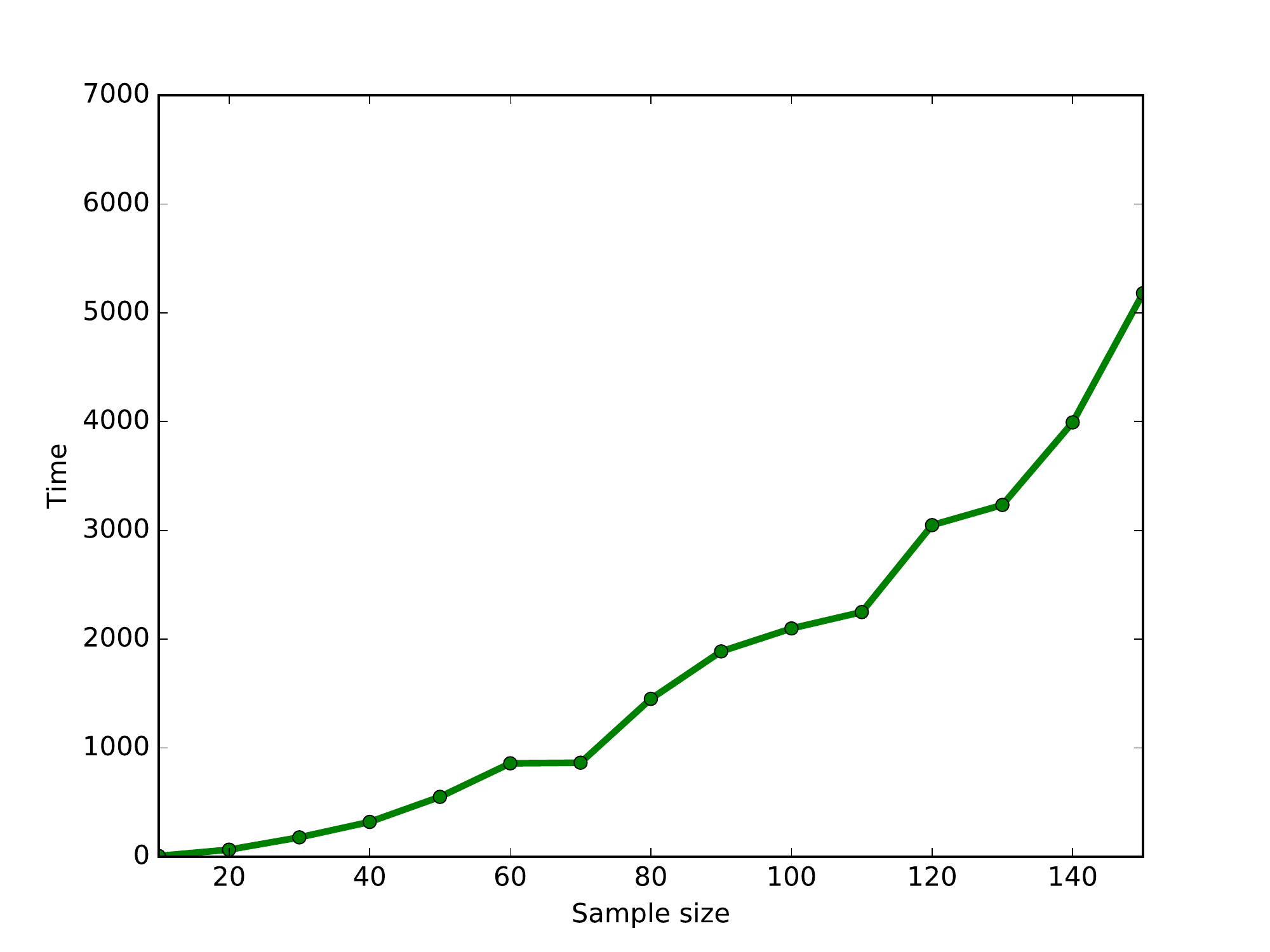} 
\label{fig:metric_learning_mnist_time}}
 \caption{Average training time (in seconds) with respect to the sample size $p$.}
 \label{fig:metric_learning_time}
\end{figure}


Figure~\ref{fig:metric_learning_synthetic} shows the test risk of the ERM with respect to the sample size $p$ for both sampling strategies on the synthetic data set. As predicted by our theoretical analysis, the incomplete $U$-statistic strategy achieves a significantly smaller risk on average. For instance, it gets within 5\% error of the true risk minimizer for $p = 50$, while the complete $U$-statistic needs $p > 80$ to reach the same performance. This represents twice more computational time, as shown in Figure~\ref{fig:metric_learning_synthetic_time} (as expected, the runtime increases roughly quadratically with $p$). The incomplete $U$-statistic strategy also has the advantage of having a much smaller variance between the runs, which makes it more reliable.
The same conclusions hold for the MNIST data set, as can be seen in Figure~\ref{fig:metric_learning_mnist} and Figure~\ref{fig:metric_learning_mnist_time}. 

\subsubsection{Stochastic Gradient Descent}
\label{sec:expsgd}

In this section, we focus on solving the ERM problem \eqref{eq:ermexp} using Stochastic Gradient Descent and compare two approaches (analyzed in Section~\ref{sec:SGD}) to construct a mini-batch at each iteration. The first strategy, SGD-Complete, is to randomly draw (with replacement) a subsample and use the complete $U$-statistic associated with the subsample as the gradient estimate. The second strategy, SGD-Incomplete (the one we promote in this paper), consists in sampling an incomplete $U$-statistic with the same number of terms as in SGD-Complete.

For this experiment, we use the MNIST data set. We set the threshold in \eqref{eq:riskexp} to $b=2$ and the learning rate of SGD at iteration $t$ to $\eta_t=1/(\eta_0t)$ where $\eta_0\in\{1,2.5,5,10,25,50\}$. To reduce computational cost, we only project our solution onto the PSD cone at the end of the algorithm, following the ``one projection'' principle used by \citet{Chechik2010a}. We try several values $m$ for the mini-batch size, namely $m\in\{10,28,55,105,253\}$.\footnote{For each $m$, we can construct a complete $U$-statistic from $n'$ samples with $n'(n'-1)/2=m$ terms.}
For each mini-batch size, we run SGD for 10,000 iterations and select the learning rate parameter $\eta_0$ that achieves the minimum risk on 100,000 pairs randomly sampled from the training set. We then estimate the generalization risk using 100,000 pairs randomly sampled from the test set.

\begin{figure}[t]
\centering
\subfigure{\includegraphics[width=0.33\columnwidth]{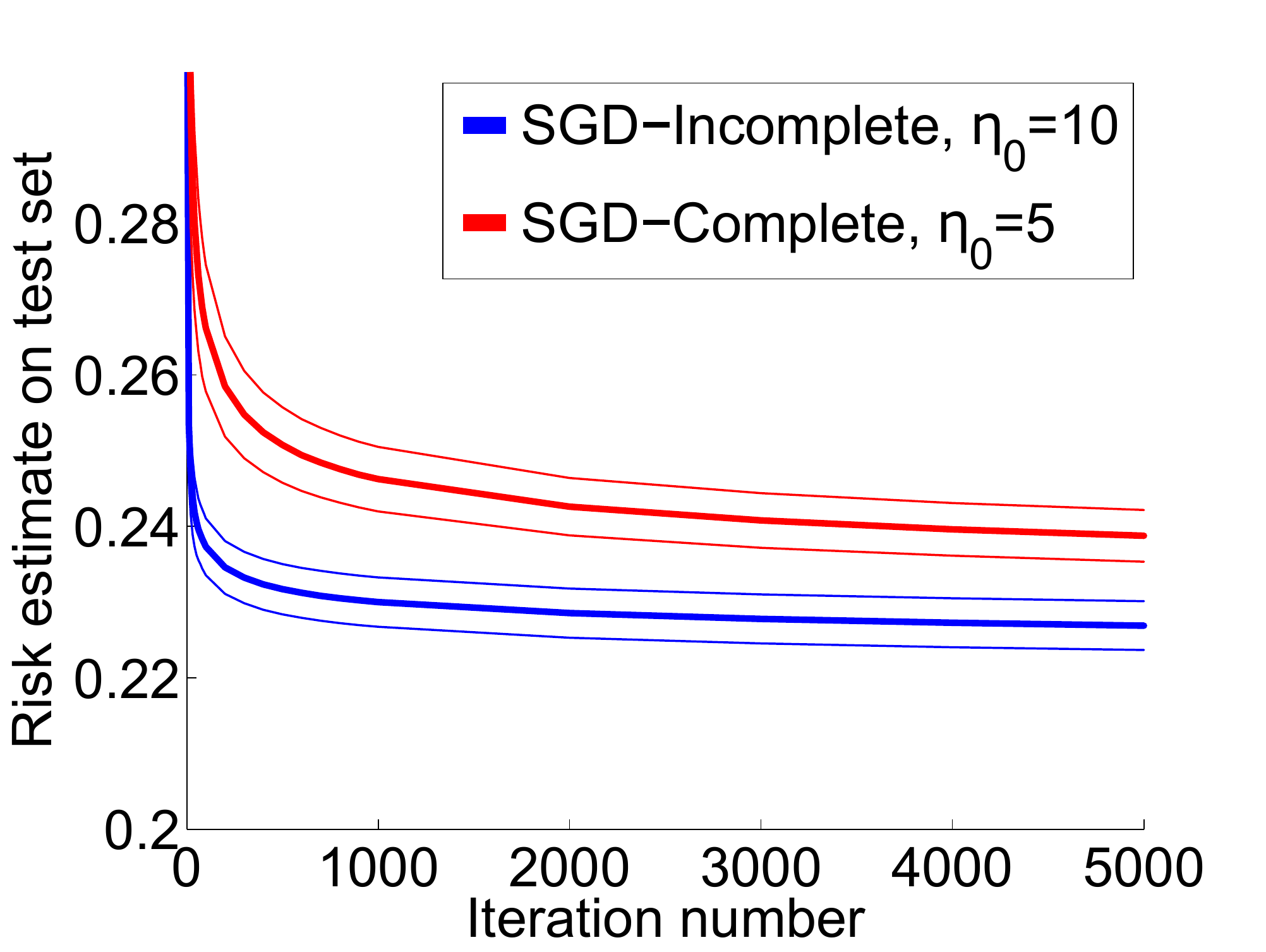}
}
\hspace*{-0.3cm}\subfigure{\includegraphics[width=0.33\columnwidth]{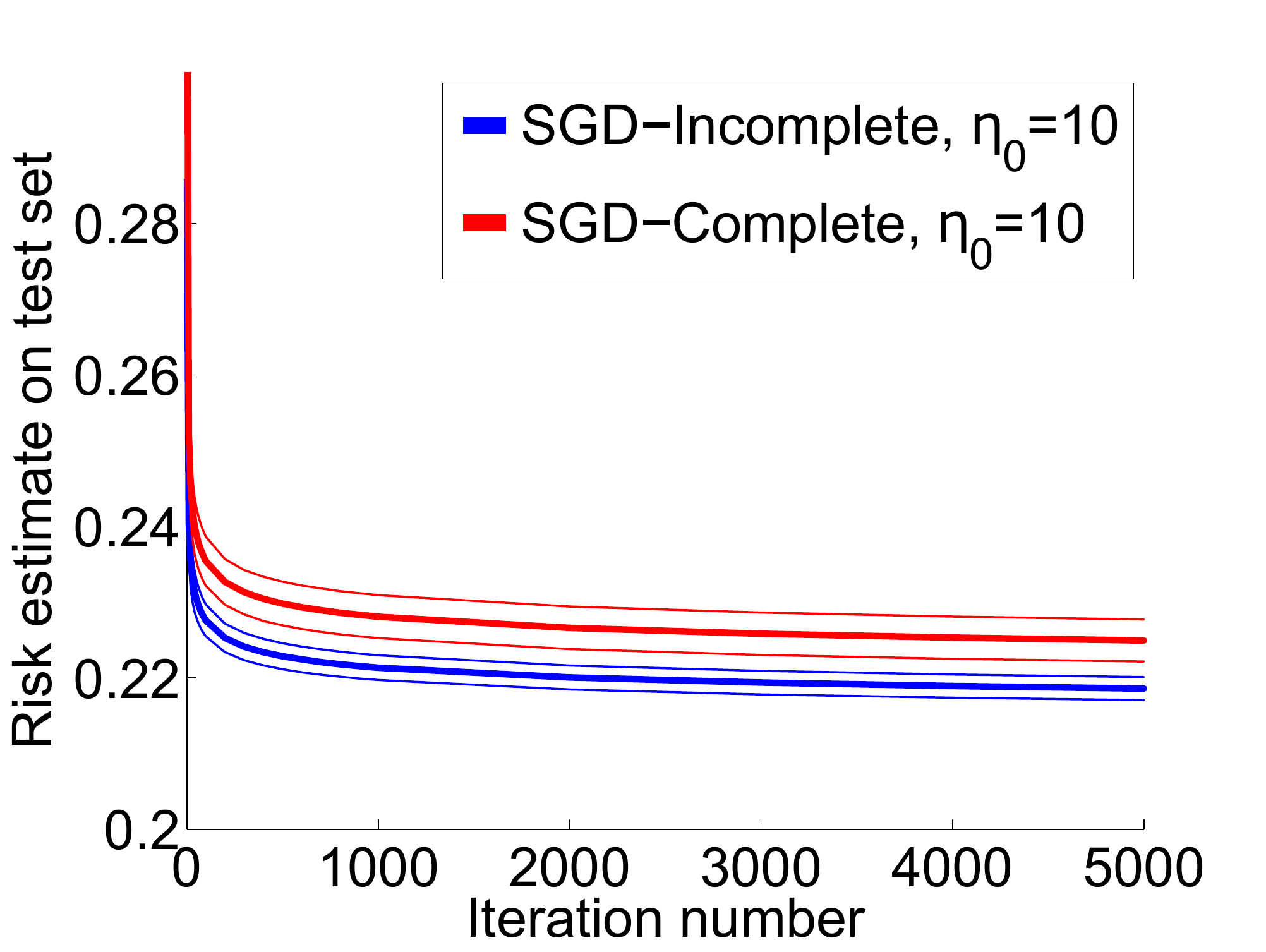}
}
\hspace*{-0.3cm}\subfigure{\includegraphics[width=0.33\columnwidth]{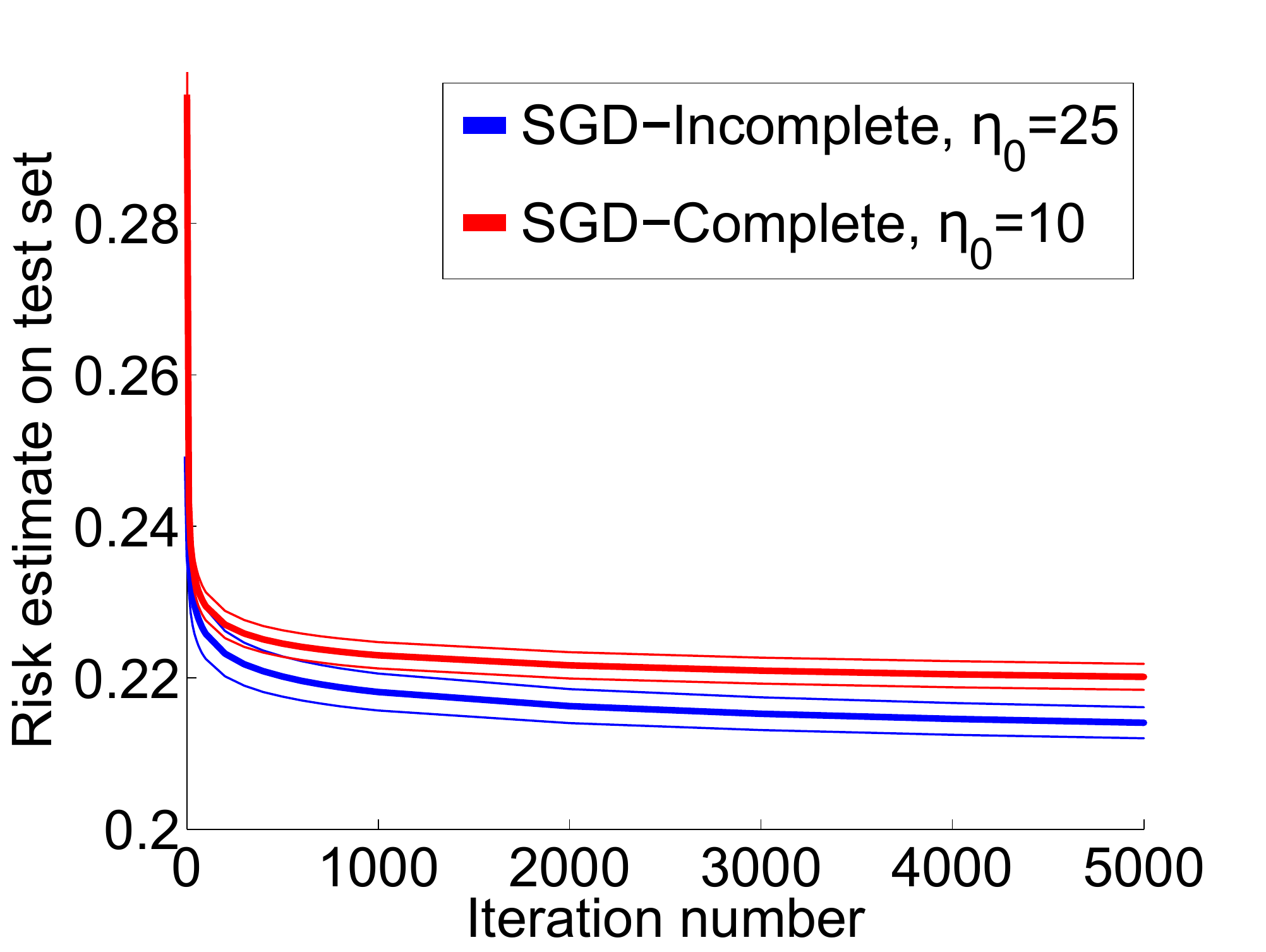}
}
\addtocounter{subfigure}{-3}
\subfigure[$m=10$]{\includegraphics[width=0.33\columnwidth]{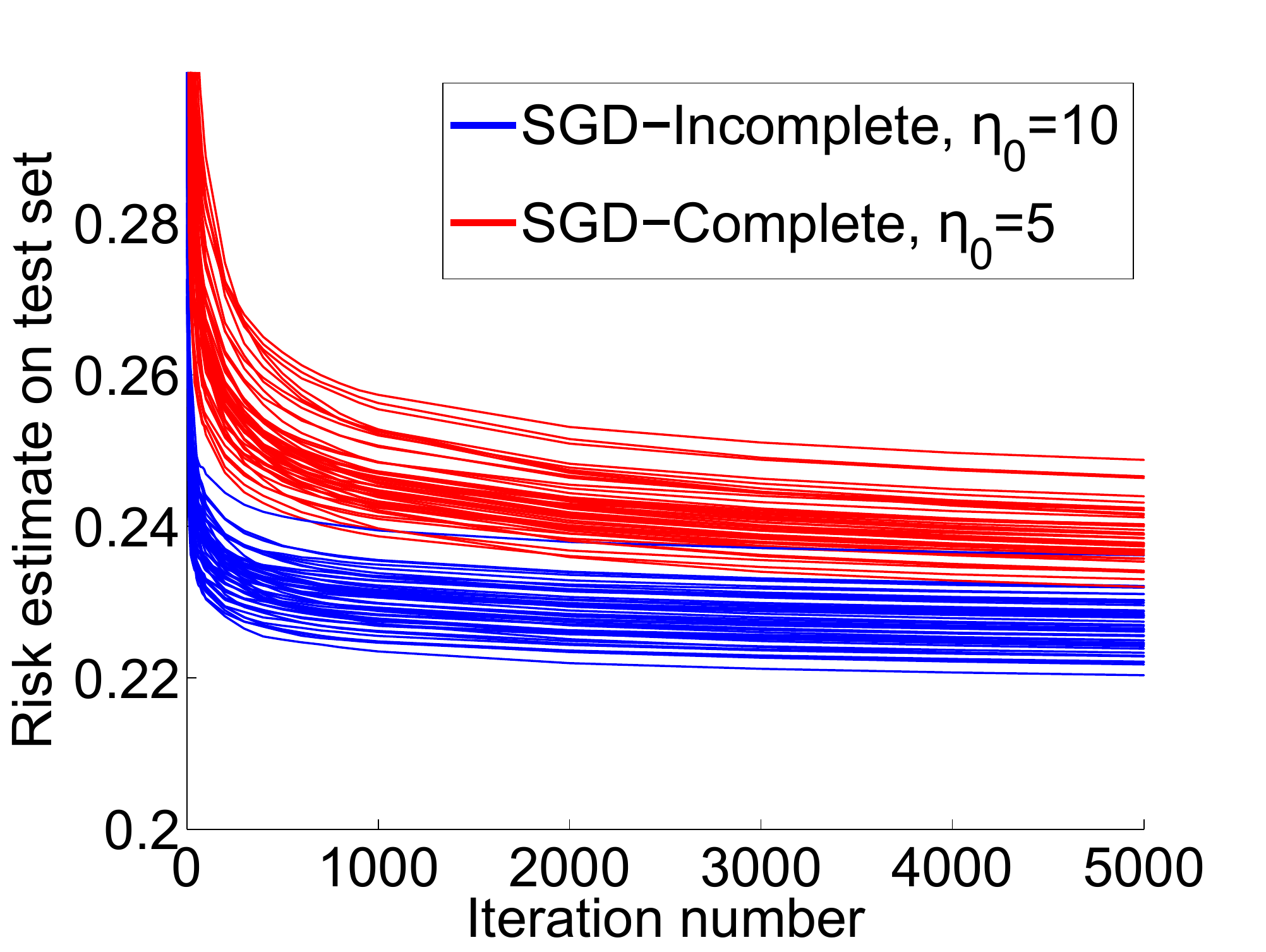}
\label{m10ind}}
\hspace*{-0.3cm}\subfigure[$m=55$]{\includegraphics[width=0.33\columnwidth]{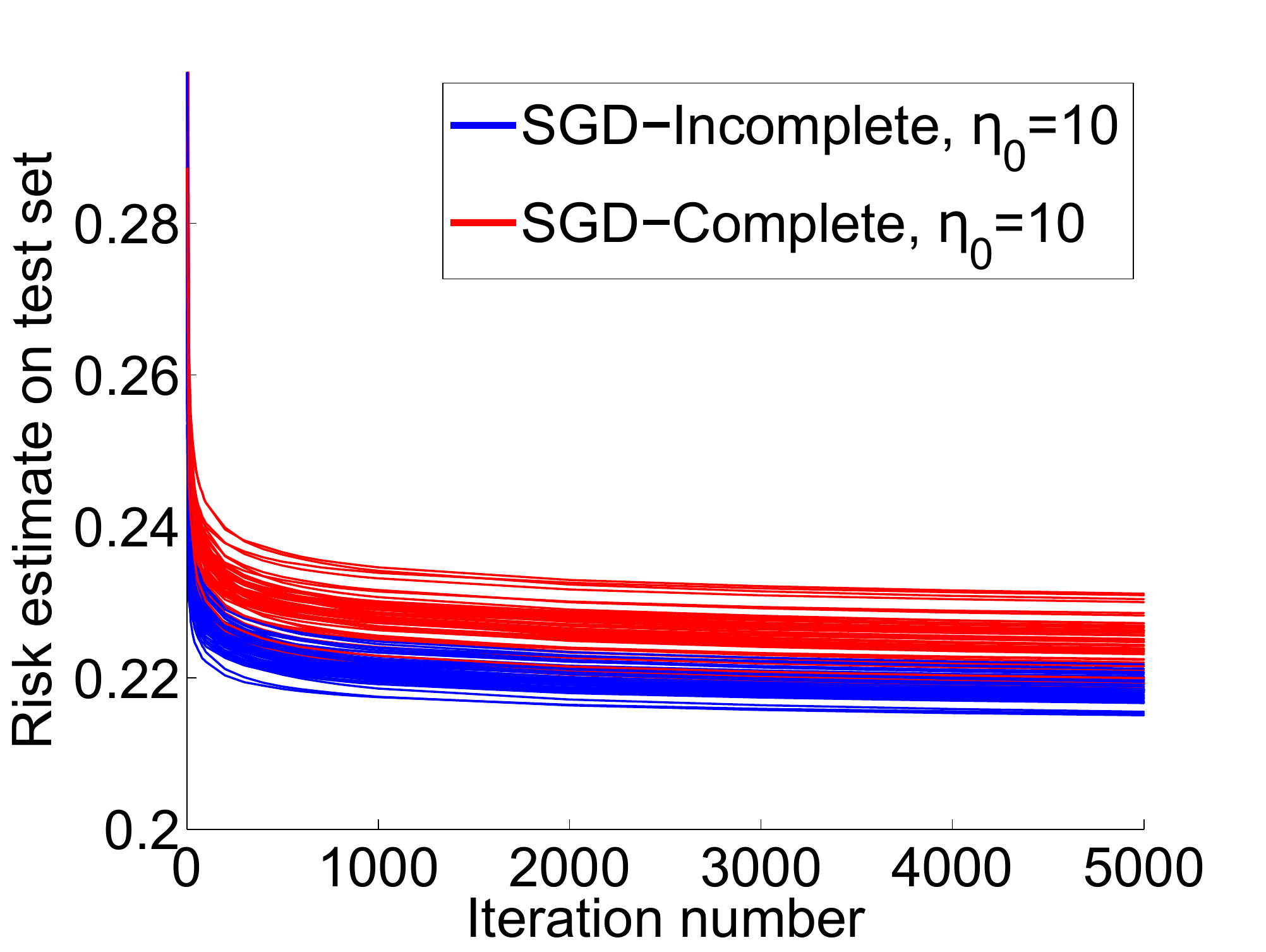} 
\label{m55ind}}
\hspace*{-0.3cm}\subfigure[$m=253$]{\includegraphics[width=0.33\columnwidth]{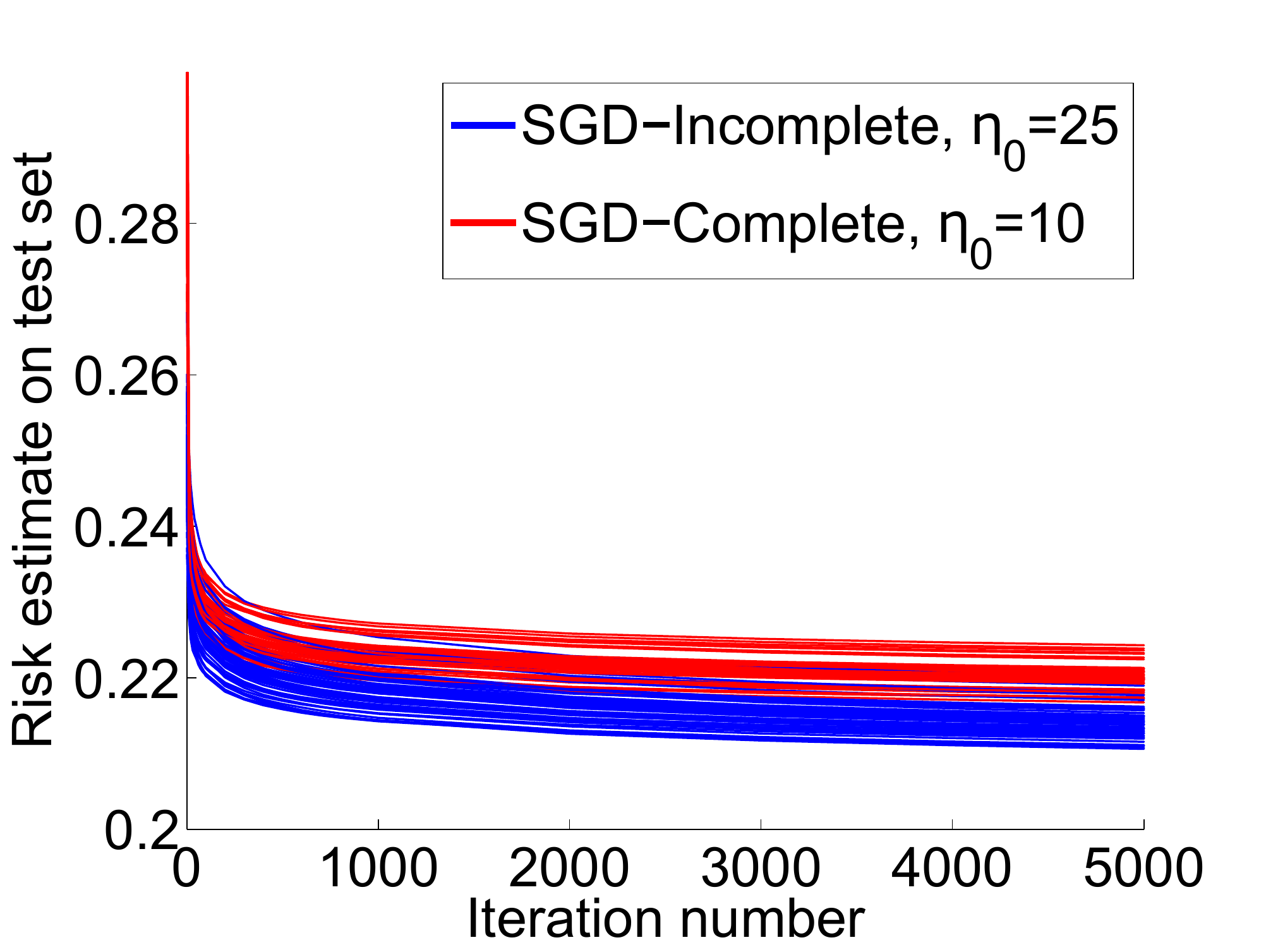} 
\label{m253ind}}
 \caption{SGD results on the MNIST data set for various mini-batch size $m$. The top row shows the means and standard deviations over 50 runs, while the bottom row shows each run separately.}
 \label{fig:sgd_mnist}
\end{figure}

For all mini-batch sizes, SGD-Incomplete achieves significantly better test risk than SGD-Complete. Detailed results are shown in Figure~\ref{fig:sgd_mnist} for three mini-batch sizes, where we plot the evolution of the test risk with respect to the iteration number.\footnote{We point out that the figures look the same if we plot the runtime instead of the iteration number. Indeed, the time spent on computing the gradients (which is the same for both variants) largely dominates the time spent on the random draws.} We make several comments. First, notice that the best learning rate is often larger for SGD-Incomplete than for SGD-Complete ($m=10$ and $m=253$). This confirms that gradient estimates from the former strategy are generally more reliable. This is further supported by the fact that even though larger learning rates increase the variance of SGD, in these two cases SGD-Complete and SGD-Incomplete have similar variance. On the other hand, for $m=55$ the learning rate is the same for both strategies. SGD-Incomplete again performs significantly better on average and also has smaller variance. Lastly, as one should expect, the gap between SGD-Complete and SGD-Incomplete reduces as the size of the mini-batch increases. Note however that in practical implementations, the relatively small mini-batch sizes (in the order of a few tens or hundreds) are generally those which achieve the best error/time trade-off.

\subsection{Model Selection in Clustering}

In this section, we are interested in the clustering problem described in Section~\ref{sec:clustering}. Specifically, let $X_1,\; \ldots,\; X_n\in\mathbb{R}^d$ be the set of points to be clustered. Let the clustering risk associated with a partition $\mathcal{P}$ into $M$ groups $\mathcal{C}_1,\dots,\mathcal{C}_M$ be:
\begin{equation}\label{eq:exp_clust_risk}
\widehat{W}_{n}(\Pcal)=\frac{2}{n(n-1)}\sum_{m=1}^{M}\sum_{1\leq i<j \leq n}D(X_{i},X_{j})\cdot\mathbb{I}\{(X_i,X_j)\in \mathcal{C}_m^2  \}.
\end{equation}
In this experiment, given a set of candidate partitions, we want to perform model selection by picking the partition which minimizes the risk \eqref{eq:exp_clust_risk} plus some term penalizing the complexity of the partition. When the number of points $n$ is large, the complete risk is very expensive to compute. Our strategy is to replace it with an incomplete approximation with much fewer terms. Like in the approach theoretically investigated in  Section~\ref{subsec:modelselect}, the goal here is to show that using the incomplete approximation instead of the complete version as the goodness-of-fit measure in a complexity penalized criterion does not damage the selection, while reducing the computational cost. For simplicity, the complexity penalty we use below is not of the same type as the structural {\sc VC} dimension-based penalty considered in Theorem \ref{thm:selec}, but we will see that the incomplete approximation is very accurate and can thus effectively replace the complete version regardless of the penalty used.

The experimental setup is as follows. We used the forest cover type data set,\footnote{\url{https://archive.ics.uci.edu/ml/datasets/Covertype}} which is popular to benchmark clustering algorithms \citep[see for instance][]{Kanungo2004a}. To be able to evaluate the complete risk, we work with $n=5,000$ points subsampled at random from the entire data set of 581,012 points in dimension 54. We then generated a hierarchical clustering of these points using agglomerative clustering with Ward's criterion \citep{Ward1963a} as implemented in the \texttt{scikit-learn} Python library \citep{Pedregosa2011a}. This defines $n$ partitions $\mathcal{P}_1,\dots,\mathcal{P}_n$ where $\mathcal{P}_m$ consists of $m$ clusters ($\mathcal{P}_1$ corresponds to a single cluster containing all points, while in $\mathcal{P}_n$ each point has its own cluster).

\begin{figure}[t]
\centering
\subfigure[Risk]{\includegraphics[width=0.48\columnwidth]{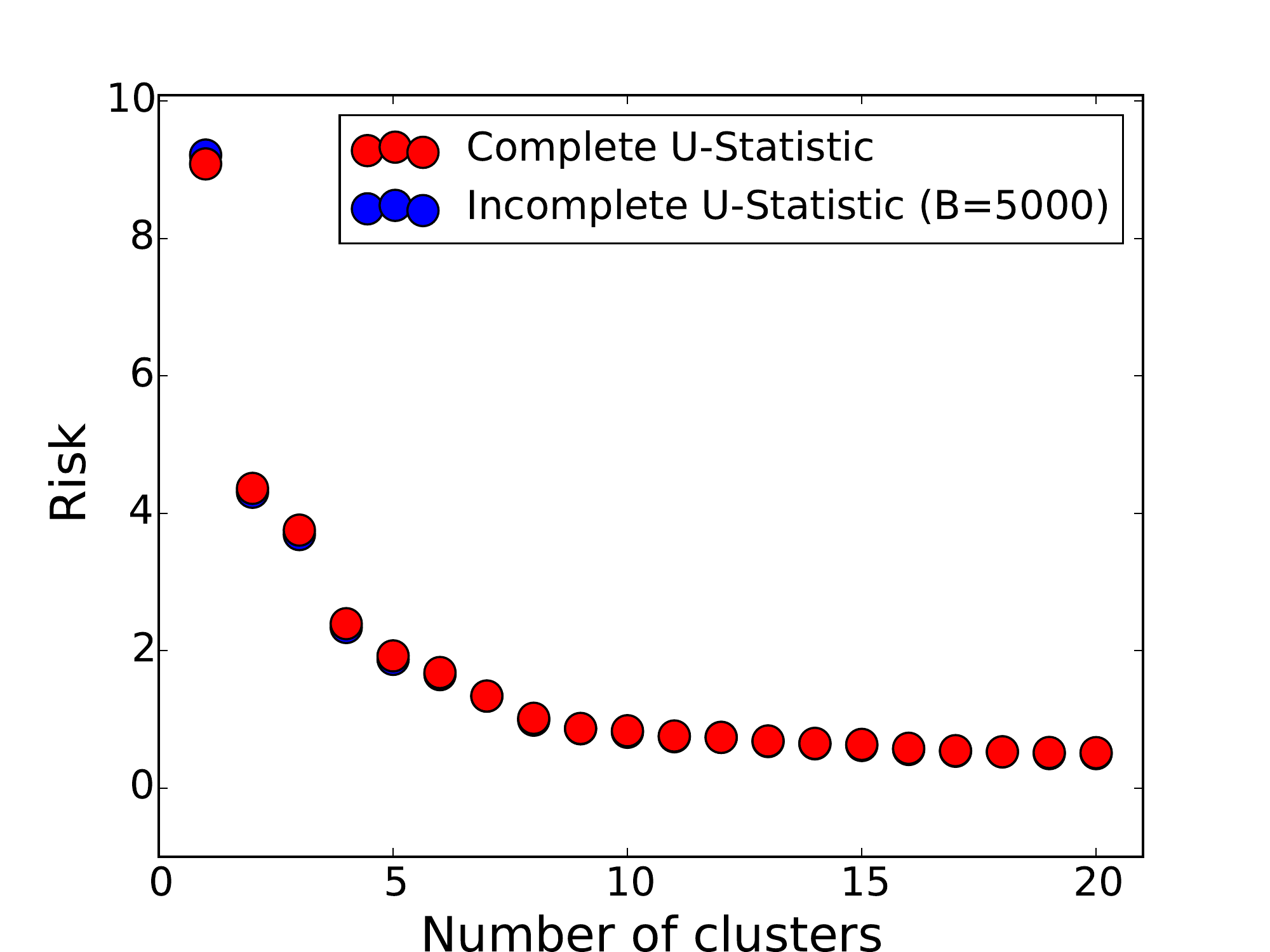}
\label{fig:model_select_risk}}
\subfigure[Penalized risk]{\includegraphics[width=0.48\columnwidth]{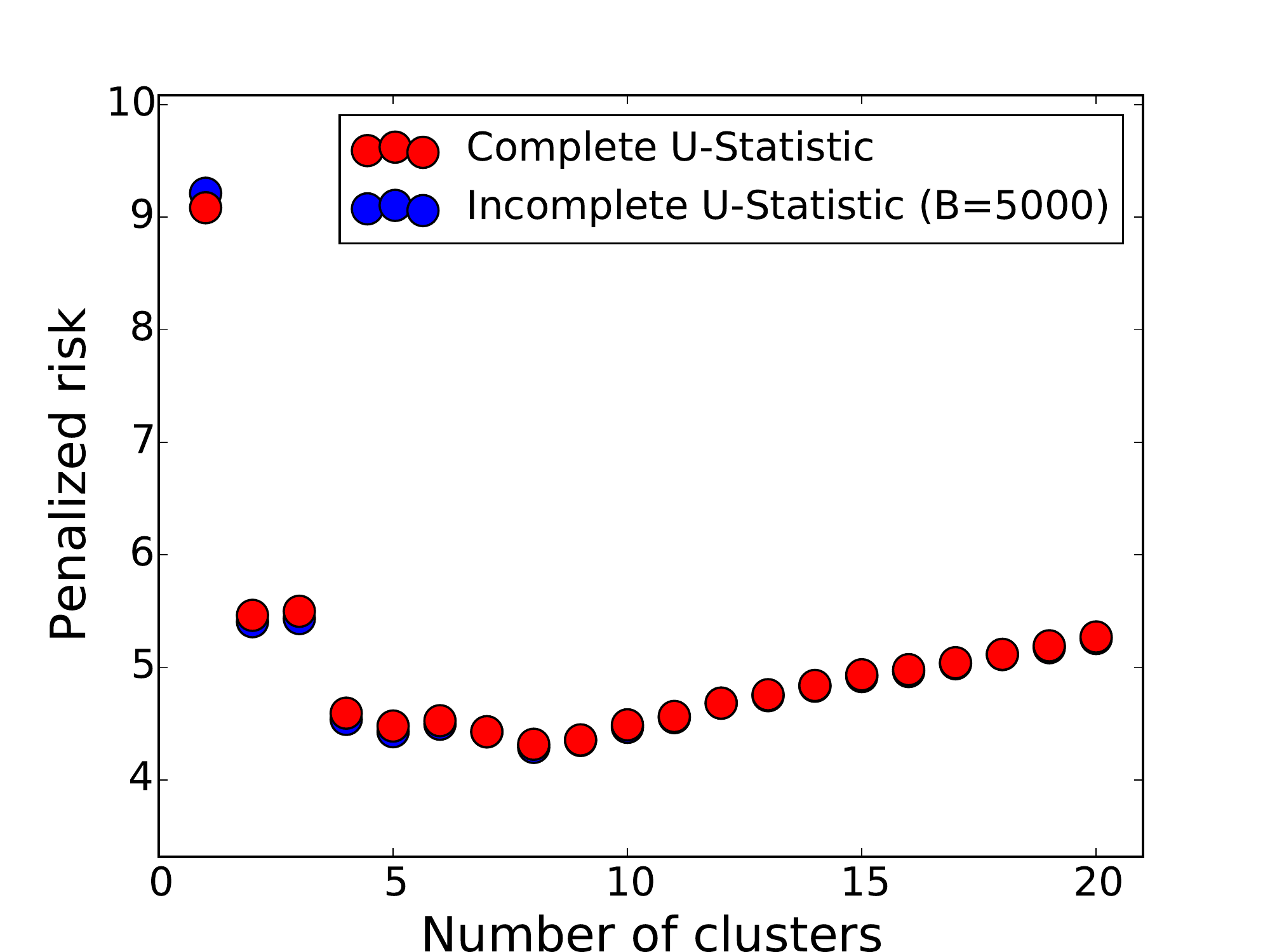} 
\label{fig:model_select_risk_pen}}
 \caption{Clustering model selection results on the forest cover type data set. Figure~\ref{fig:model_select_risk} shows the risk (complete and incomplete with $B=5,000$ terms) for the first 20 partitions, while Figure~\ref{fig:model_select_risk_pen} shows the penalized risk for $c=1.1$.}
 \label{fig:model_select}
\end{figure}

For each partition size, we first compare the value of the complete risk \eqref{eq:exp_clust_risk} with $n(n-1)= 24,995,000$ terms with that of an incomplete version with only $B=n=5,000$ pairs drawn at random. As shown in Figure~\ref{fig:model_select_risk}, the incomplete $U$-statistic is a very accurate approximation of the complete one, despite consisting of 5000 times less terms. It will thus lead to similar results in model selection. To illustrate, we use a simple penalty term of the form $\text{pen}(\mathcal{P}_m)=c\cdot\log(m)$ where $c$ is a scaling constant. Figure~\ref{fig:model_select_risk_pen} shows that both selection criteria choose the same model $\mathcal{P}_8$. Performing this model selection over $\mathcal{P}_1,\dots,\mathcal{P}_{20}$ took about 66 seconds for the complete $U$-statistic, compared to only 0.1 seconds for the incomplete version.\footnote{The $n\times n$ distance matrix was precomputed before running the agglomerative clustering algorithm. The associated runtime is thus not taken into account in these timing results.}

Finally, we generated 100 incomplete $U$-statistics with different random seeds ; all of them correctly identified $\mathcal{P}_8$ as the best model. Using $B=5,000$ pairs is thus sufficient to obtain reliable results with an incomplete $U$-statistic for this data set. In contrast, the complete $U$-statistics based on a subsample (leading to the same number of pairs) selected the correct model in only 57\% of cases.


\section{Conclusion}
\label{sec:conclu}

In a wide variety of statistical learning problems, $U$-statistics are natural estimates of the risk measure one seeks to optimize. As the sizes of the samples increase, the computation of such functionals involves summing a rapidly exploding number of terms and becomes numerically unfeasible. In this paper, we argue that for such problems, \textit{Empirical Risk Minimization} can be implemented using statistical counterparts of the risk based on much less terms (picked randomly by means of sampling with replacement), referred to as \textit{incomplete $U$-statistics}. Using a novel deviation inequality, we have shown that this approximation scheme does not deteriorate the learning rates, even preserving fast rates in certain situations where they are proved to occur. Furthermore, we have extended these results to $U$-statistics based on different sampling schemes (Bernoulli sampling, sampling without replacement) and shown how such functionals can be used for the purpose of model selection and for implementing ERM iterative procedures based on stochastic gradient descent. Beyond theoretical rate bounds, the efficiency of the approach we promote is illustrated by several numerical experiments.


\paragraph{Acknowledgements} This work is supported by the Chair ``Machine Learning for Big Data'' of T\'el\'ecom ParisTech, and was conducted while A. Bellet was affiliated with T\'el\'ecom ParisTech. The authors are grateful to the reviewers for their careful reading of the paper, which permitted to improve significantly the presentation of the results.


\appendix

\section{Proof of Proposition \ref{prop:Uproc}} 
\label{app:propUproc}
  Set $N=\min\{ \lfloor n_1/d_1\rfloor,\ldots,
  \lfloor n_K/d_K \rfloor\}$ and let
  \begin{multline*}
  V_{H}\left(X^{(1)}_1,\;\ldots,\; X^{(1)}_{n_1},\;\ldots,\; X^{(K)}_1,\;\ldots,\; X^{(K)}_{n_K}\right)=\frac{1}{N}\Big[H\left(X^{(1)}_1,\;\ldots,\; X^{(1)}_{d_1},\;\ldots,\; X^{(K)}_1,\;\ldots,\; X^{(K)}_{d_K}\right)\\
  +H\left(X^{(1)}_{d_1+1},\;\ldots,\; X^{(1)}_{2d_1},\;\ldots,\; X^{(K)}_{d_K+1},\;\ldots,\; X^{(K)}_{2d_K}\right)  +\ldots\\
   +H\left(X^{(1)}_{N d_1-d_1+1},\;\ldots,\; X^{(1)}_{N d_1},\;\ldots,\; X^{(K)}_{N d_K-d_K+1},\;\ldots,\; X^{(K)}_{N d_K}\right)\Big],
  \end{multline*}
  for any $H\in \mathcal{H}$
  Recall that the $K$-sample $U$-statistic $U_{
  \mathbf{n}}(H)$ can be expressed as
  \begin{equation}
  U_{\mathbf{n}}(H)=\frac{1}{n_1!\cdots n_K!}\sum_{\sigma_1\in \mathfrak{S}_{n_1},\;\ldots,\; \sigma_K\in \mathfrak{S}_{n_K}}V_H\left(X^{(1)}_{\sigma_1(1)},\; \ldots,\; X^{(1)}_{\sigma_1(n_1)},\;\ldots,\;  X^{(K)}_{\sigma_K(1)},\; \ldots,\; X^{(K)}_{\sigma_K(n_K)} \right),
  \end{equation}
  where $\mathfrak{S}_m$ denotes the symmetric group of order $m$ for any $m\geq 1$.
  This representation as an average of sums of $N$ independent terms is known as the (first) Hoeffding's decomposition, see \cite{Hoeffding48}. Then, using Jensen's inequality in particular, one may easily show that, for any nondecreasing convex function $\psi:\mathbb{R}_+\rightarrow \mathbb{R}$, we have:
  \begin{equation}\label{eq:tool1}
  \mathbb{E}\left[ \psi\left(\sup_{H\in \mathcal{H}}\left\vert U_{\mathbf{n}}(\bar{H})  \right\vert  \right)  \right]\leq
  \mathbb{E}\left[ \psi\left(\sup_{H\in \mathcal{H}}\left\vert V_{\bar{H}}(X^{(1)}_1,\;\ldots,\; X^{(1)}_{n_1},\;\ldots,\; X^{(K)}_1,\;\ldots,\; X^{(K)}_{n_K})  \right\vert  \right)  \right],
  \end{equation}
  where we set $\bar{H}=H-\mu(H)$ for all $H\in \mathcal{H}$.
  Now, using standard symmetrization and randomization arguments (see \cite{GineZinn} for instance) and \eqref{eq:tool1}, we obtain that
  \begin{equation}\label{eq:tool2}
  \mathbb{E}\left[ \psi\left(\sup_{H\in \mathcal{H}}\left\vert U_{\mathbf{n}}(\bar{H})  \right\vert  \right)  \right]\leq \mathbb{E}\left[ \psi\left(2\mathcal{R}_{N}\right) \right],
  \end{equation}
  where
  $$
  \mathcal{R}_{N}=\sup_{H\in\mathcal{H}}\frac{1}{N}\sum_{l=1}^{N}\epsilon_l
  H\left(X^{(1)}_{(l-1)d_1+1},\;\ldots,\; X^{(1)}_{ld_1},\;\ldots,\; X^{(K)}_{(l-1)d_K+1},\;\ldots,\; X^{(K)}_{ld_K}\right),
  $$
  is a Rademacher average based on the Rademacher chaos $\epsilon_1,\; \ldots,\; \epsilon_{N}$ (independent random symmetric sign variables), independent from the $X^{(k)}_i$'s. We now apply the bounded difference inequality (see \cite{McDiarmid}) to the functional $\mathcal{R}_{N}$, seen as a function of the i.i.d. random variables $(\epsilon_l, X^{(1)}_{(l-1)d_1+1},\;\ldots,\; X^{(1)}_{ld_1},\;\ldots,\; X^{(K)}_{(l-1)d_K+1},\;\ldots,\; X^{(K)}_{ld_K})$, $1\leq l\leq N$: changing any of these random variables change the value of $\mathcal{R}_{N}$ by at most $\mathcal{M}_{\mathcal{H}}/N$. One thus obtains from \eqref{eq:tool2} with $\psi(x)=\exp(\lambda x)$, where $\lambda>0$ is a parameter which shall be chosen later, that:
  \begin{equation}
  \mathbb{E}\left[ \exp\left(\lambda\sup_{H\in \mathcal{H}}\left\vert U_{\mathbf{n}}(\bar{H})  \right\vert  \right)  \right]\leq \exp \left(2\lambda\mathbb{E}[\mathcal{R}_{N}] +\frac{\mathcal{M}^2_{\mathcal{H}}\lambda^2}{4N} \right).
  \end{equation}
  Applying Chernoff's method, one then gets:
  \begin{equation}\label{eq:tool3}
  \mathbb{P}\left\{ \sup_{H\in \mathcal{H}}\left\vert U_{\mathbf{n}}(\bar{H}) \right\vert>\eta \right\} \leq \exp\left(-\lambda\eta +2\lambda\mathbb{E}[\mathcal{R}_{N}]+\frac{\mathcal{M}^2_{\mathcal{H}}\lambda^2}{4N}\right).
  \end{equation}
  Using the bound (see Eq. (6) in \cite{BBL05} for instance)
  $$
  \mathbb{E}[\mathcal{R}_{N}]\leq \mathcal{M}_{\mathcal{H}}\sqrt{\frac{2V\log(1+N)}{N}}
  $$
  and taking $\lambda=2N (\eta-2\mathbb{E}[\mathcal{R}_{N}])/\mathcal{M}^2_{\mathcal{H}}$ in \eqref{eq:tool3}, one finally establishes the desired result. 
%

\section{Proof of Theorem \ref{thm:main}}
\label{app:thmmain}

For convenience, we introduce the random sequence $%
\zeta=((\zeta_k(I))_{I\in \Lambda})_{1\leq k\leq B}$, where $%
\zeta_k(I)$ is equal to $1$ if the tuple $I=(I_1,\; \ldots,\; I_K)$ has
been selected at the $k$-th draw and to $0$ otherwise: the $\zeta_k$'s
are i.i.d. random vectors and, for all $(k,I)\in\{1,\;\ldots,\; B\}\times
\Lambda$, the r.v. $\zeta_k(I)$ has a Bernoulli distribution with
parameter $1/\#\Lambda$. We also set $\mathbf{X}_I=(\mathbf{X}%
^{(1)}_{I_1},\; \ldots,\; \mathbf{X}^{(K)}_{I_K})$ for any $I$ in $\Lambda$.
Equipped with these notations, observe first that one may write: $\forall
B\geq 1$, $\forall \mathbf{n}\in \mathbb{N}^{*K}$, 
\begin{equation*}
\widetilde{U}_B(H)-U_{\mathbf{n}}(H)=\frac{1}{B}\sum_{k=1}^B \mathcal{Z}%
_k(H),
\end{equation*}
where $\mathcal{Z}_k(H)=\sum_{I\in \Lambda} ( \zeta_k(I)-1/\# \Lambda )H(%
\mathbf{X}_I)$ for any $(k,I)\in\{1,\;\ldots,\; B\}\times \Lambda$. It
follows from the independence between the $\mathbf{X}_I$'s and the $%
\zeta(I)$'s that, for all $H\in \mathcal{H}$, conditioned upon the $%
\mathbf{X}_I$'s, the variables $\mathcal{Z}_1(H),\; \ldots,\; \mathcal{Z}%
_B(H)$ are independent, centered and almost-surely bounded by $2\mathcal{M}_{%
\mathcal{H}}$ (notice that $\sum_{I\in\Lambda}\zeta_k(I)=1$ for all $%
k\geq1$). By virtue of Sauer's lemma, since $\mathcal{H}$ is a {\sc VC} major class with finite {\sc VC} dimension $V$, we have, for fixed $\mathbf{X}_I$'s:
$$
\#\{(H(\mathbf{X}_I))_{I\in \Lambda}:\; H\in \mathcal{H}\}\leq (1+\#\Lambda)^V.
$$
Hence, conditioned upon the $\mathbf{X}_I$'s, using the union bound and next Hoeffding's inequality applied to the independent sequence $\mathcal{Z}_1(H),\;\ldots,\; \mathcal{Z}_B(H)$, for
all $\eta>0$, we obtain that: 
\begin{eqnarray*}
\mathbb{P}\left\{ \sup_{H\in \mathcal{H}} \left\vert \widetilde{U}_B(H)-U_{%
\mathbf{n}}(H) \right\vert>\eta\;\mid\; (\mathbf{X}_I)_{I\in \Lambda} \right\}
&\leq& \mathbb{P}\left\{\sup_{H\in \mathcal{H}} \left\vert \frac{1}{B}\sum_{k=1}^B\mathcal{Z}%
_k(H) \right\vert >\eta\; {\mid}\; (\mathbf{X}_I)_{I\in \Lambda}\right\} \\
&\leq&  2(1+\#\Lambda)^Ve^{-B\eta^2/\left(2\mathcal{M}^2_{\mathcal{H}}\right)}.
\end{eqnarray*}
Taking the expectation, this proves the first assertion of the theorem. Notice that this can be formulated: for any $\delta\in (0,1)$, we have with probability at least $1-\delta$:
\begin{equation*}
 \sup_{H\in \mathcal{H}} \left\vert \widetilde{U}_B(H)-U_{
\mathbf{n}}(H) \right\vert\leq \mathcal{M}_{\mathcal{H}}\times
\sqrt{2\frac{V\log(1+\# \Lambda)+\log (2/\delta)}{B}}.
\end{equation*}
 
 Turning to the second part of the theorem, it straightforwardly results from the first part combined with Proposition \ref{prop:Uproc}.

\section{Proof of Corollary \ref{cor}}\label{subsec:proofcor}

Assertion $(i)$ is a direct application of Assertion $(ii)$ in Theorem \ref{thm:main} combined with the bound
$\mu(\widehat{H}_{B}) - \inf_{H\in \mathcal{H}}   \mu(H)  \leq 2 \sup_{H \in \mathcal{H}}\vert \widetilde{U}_{B}(H) -  \mu(H)\vert$.

Turning next to Assertion $(ii)$, observe that by triangle inequality we have:
\begin{equation}\label{eq:triangular}
\mathbb{E}\left[ \sup_{H\in \mathcal{H}_m}\vert \widetilde{U}_B(H)-\mu(H)  \vert  \right]\leq \mathbb{E}\left[ \sup_{H\in \mathcal{H}_m}\vert \widetilde{U}_B(H)-U_{\mathbf{n}}(H)  \vert  \right]+\mathbb{E}\left[ \sup_{H\in \mathcal{H}_m}\vert U_{\mathbf{n}}(H)-\mu(H)  \vert  \right].
\end{equation}
The same argument as that used in Theorem \ref{thm:main} (with $\psi(u)=u$ for any $u\geq 0$) yields a bound for the second term on the right hand side of Eq. \eqref{eq:triangular}:
\begin{equation}\label{eq:expec1}
\mathbb{E}\left[ \sup_{H\in \mathcal{H}}\vert U_{\mathbf{n}}(H)-\mu(H)  \vert  \right]\leq 2\mathcal{M}_{\mathcal{H}}\sqrt{\frac{2V\log (1+N)}{N}}.
\end{equation}
The first term can be controlled by means of the following lemma, whose proof can be found for instance in \citet[][Lemmas 1.2 and 1.3]{Gabor}. 
\begin{lemma} The following assertions hold true.
\begin{itemize}
\item[(i)] Hoeffding's lemma. Let $Z$ be an integrable r.v. with mean zero such that $a\leq Z \leq b$ almost-surely. Then, we have: $\forall s>0$
$$
\mathbb{E}[\exp(sZ)]\leq \exp\left(s^2(b-a)^2/8\right).
$$
\item[(ii)] Let $M\geq 1$ and $Z_1,\; \ldots,\; Z_M$ be real valued random variables. Suppose that there exists $\sigma>0$ such that $\forall s \in \mathbb{R}$: $\mathbb{E}[\exp(sZ_i)]\leq e^{s^2\sigma^2/2}$ for all $i\in \{1,\; \ldots,\; M \}$. Then, we have:
\begin{equation}
\mathbb{E}\left[ \max_{1\leq i \leq M}\vert Z_i\vert \right]\leq \sigma \sqrt{2\log (2M)}.
\end{equation}
\end{itemize}
\end{lemma}
Assertion $(i)$ shows that, since $-\mathcal{M}_{\mathcal{H}}\leq \mathcal{Z}_k(H)\leq \mathcal{M}_{\mathcal{H}}$ almost surely,
$$
\mathbb{E}\left[ \exp(s\sum_{k=1}^B\mathcal{Z}_k(H))\mid (\mathbf{X}_I)_{I\in \Lambda} \right]\leq  e^{\frac{1}{2}Bs^2\mathcal{M}_{\mathcal{H}}^2}.
$$
With $\sigma=\mathcal{M}_{\mathcal{H}}\sqrt{B}$ and $M=\#\{ H(\mathbf{X}_I):\;  H\in \mathcal{H} \}\leq (1+\#\Lambda)^{V}$, conditioning upon $(\mathbf{X}_{I})_{I\in \Lambda}$, this result yields:
\begin{equation}
\mathbb{E}\left[ \sup_{H\in \mathcal{H}}\left\vert \frac{1}{B}\sum_{k=1}^B \mathcal{Z}_k(H) \right\vert \mid (\mathbf{X}_{I})_{I\in \Lambda} \right]\leq \mathcal{M}_{\mathcal{H}}\sqrt{\frac{2(\log 2 + V\log (1+\#\Lambda))}{B}}.
\end{equation}
Integrating next over $(\mathbf{X}_{I})_{I\in \Lambda}$ and combining the resulting bound with \eqref{eq:triangular} and \eqref{eq:expec1} leads to the inequality stated in $(ii)$.
\medskip

\noindent {\bf A bound for the expected value.} For completeness, we point out that the expected value of $\sup_{H\in \mathcal{H}}\vert (1/B)\sum_{k=1}^B \mathcal{Z}_k(H) \vert$ can also be bounded by means of classical symmetrization and randomization devices. Considering a "ghost" i.i.d. sample $\zeta'_1,\; \ldots,\; \zeta'_B$ independent from $((\mathbf{X}_I)_{I\in \Lambda}, \zeta)$, distributed as $\zeta$, Jensen's inequality yields:
\begin{eqnarray*}
\mathbb{E}\left[ \sup_{H\in \mathcal{H}}\left\vert \frac{1}{B}\sum_{k=1}^B \mathcal{Z}_k(H) \right\vert  \right]&=& \mathbb{E}\left[ \sup_{H\in \mathcal{H}}\left\{  \mathbb{E}\left[\left\vert \frac{1}{B}\sum_{k=1}^B \sum_{I\in \Lambda}H(\mathbf{X}_I) \left(\zeta_k(I)-\zeta'_k(I)  \right)\right\vert \mid (\mathbf{X}_I)_{I\in \Lambda}\right]\right\} \right]\\
&\leq & \mathbb{E}\left[ \sup_{H\in \mathcal{H}} \left\vert \frac{1}{B}\sum_{k=1}^B \sum_{I\in \Lambda}H(\mathbf{X}_I) \left(\zeta_k(I)-\zeta'_k(I)  \right)\right\vert  \right].
\end{eqnarray*}
Introducing next independent Rademacher variables $\epsilon_1,\; \ldots,\; \epsilon_B$, independent from $((\mathbf{X}_I)_{I\in \Lambda}, \zeta,\zeta')$, we have:
\begin{multline*}
\mathbb{E}\left[ \sup_{H\in \mathcal{H}} \left\vert \frac{1}{B}\sum_{k=1}^B \sum_{I\in \Lambda}H(\mathbf{X}_I) \left(\zeta_k(I)-\zeta'_k(I)  \right)\right\vert  \mid (\mathbf{X}_I)_{I\in \Lambda} \right]=\\
 \mathbb{E}\left[ \sup_{H\in \mathcal{H}} \left\vert \frac{1}{B}\sum_{k=1}^B \epsilon_k \sum_{I\in \Lambda}H(\mathbf{X}_I) \left(\zeta_k(I)-\zeta'_k(I)  \right)\right\vert \mid (\mathbf{X}_I)_{I\in \Lambda} \right]\\
\leq  2 \mathbb{E}\left[ \sup_{H\in \mathcal{H}} \left\vert \frac{1}{B}\sum_{k=1}^B \epsilon_k \sum_{I\in \Lambda}H(\mathbf{X}_I)\zeta_k(I)\right\vert \mid (\mathbf{X}_I)_{I\in \Lambda} \right].
\end{multline*}
We thus obtained:
$$
\mathbb{E}\left[ \sup_{H\in \mathcal{H}}\left\vert \frac{1}{B}\sum_{k=1}^B \mathcal{Z}_k(H) \right\vert  \right]\leq 
2 \mathbb{E}\left[ \sup_{H\in \mathcal{H}} \left\vert \frac{1}{B}\sum_{k=1}^B \epsilon_k \sum_{I\in \Lambda}H(\mathbf{X}_I)\zeta_k(I)\right\vert  \right].$$

\section{Proof of Theorem \ref{thm:selec}}
\label{app:thmselec}

We start with proving the intermediary result, stated below.
\begin{lemma}\label{lem:proba}
Under the assumptions stipulated in Theorem \ref{thm:selec}, we have: $\forall m\geq 1$, $\forall \epsilon>0$,
\begin{multline*}
\mathbb{P}\left\{\sup_{H\in \mathcal{H}_m}\vert \mu(H)-\widetilde{U}_B(H)\vert  > 2\mathcal{M}_{\mathcal{H}_m}\left\{  \sqrt{\frac{2V_m\log (1+N)}{N}}+\sqrt{\frac{2(\log 2 + V_m\log (1+\#\Lambda))}{B}}\right\}+ \epsilon \right\}\\
\leq \exp\left(-B^2\epsilon^2/\left(2(B+n)\mathcal{M}^2_{\mathcal{H}_m}\right)\right).
\end{multline*}
\end{lemma}
\begin{proof}
This is a direct application of the bounded difference inequality (see \cite{McDiarmid}) applied to the quantity $\sup_{H\in \mathcal{H}_m}\vert \mu(H)-\widetilde{U}_B(H)\vert$, viewed as a function of the $(B+n)$ independent random variables $(X_1^{(1)},\; X_{n_K}^{(K)}, \epsilon_1,\; \ldots,\; \epsilon_B)$ (jumps being bounded by $2\mathcal{M}_H/B$), combined with Assertion $(ii)$ of Corollary \ref{cor}.
\end{proof}
Let $m\geq 1$ and decompose the expected excess of risk of the rule picked by means of the complexity regularized incomplete $U$-statistic criterion as follows:
\begin{multline}\label{eq:decomp}
\mathbb{E}\left[\mu(\widehat{H}_{B,\widehat{m}})-\mu_m^*\right]=   \mathbb{E}\left[\mu(\widehat{H}_{B,\widehat{m}})-\widetilde{U}_B(\widehat{H}_{B,\widehat{m}})-\text{pen}(B,\widehat{m})   \right] \nonumber \\ +  \mathbb{E}\left[\inf_{j\geq 1}\left\{\widetilde{U}_B(\widehat{H}_{B,j})+\text{pen}(B,j)  \right\}   -\mu^*_m \right],
\end{multline}
where we set $\mu^*_m=\inf_{H\in \mathcal{H}_m}\mu(H)$. In order to bound the first term on the right hand side of the equation above, observe that we have: $\forall \epsilon>0$,
\begin{multline*}
\mathbb{P}\left\{\mu(\widehat{H}_{B,\widehat{m}})-\widetilde{U}_B(\widehat{H}_{B,\widehat{m}})-\text{pen}(B,\widehat{m})     >\epsilon  \right\}\leq  \mathbb{P}\left\{\sup_{j\geq 1}\left\{ \mu(\widehat{H}_{B,j})-\widetilde{U}_B(\widehat{H}_{B,j})-\text{pen}(B,j)\right\}     >\epsilon  \right\}\\
\leq \sum_{j\geq 1}\mathbb{P}\left\{\mu(\widehat{H}_{B,j})-\widetilde{U}_B(\widehat{H}_{B,j})-\text{pen}(B,j)     >\epsilon  \right\}\\
\leq\sum_{j\geq 1} \mathbb{P}\left\{\sup_{H\in \mathcal{H}_j}\vert \mu(\widehat{H})-\widetilde{U}_B(H)\vert   -\text{pen}(B,j)      >\epsilon  \right\}\\
\leq\sum_{j\geq 1} \exp\left( -\frac{-B^2}{2(B+n)\mathcal{M}^2}\left(\epsilon+2\mathcal{M}\sqrt{\frac{(B+n)\log j}{B^2}}\right)^2 \right)\\ \leq 
\exp\left( -\frac{B^2\epsilon^2}{2(B+n)\mathcal{M}^2} \right)\sum_{j\geq 1}1/j^{2}\leq 2 \exp\left( -\frac{B^2\epsilon^2}{2(B+n)\mathcal{M}^2} \right),
\end{multline*}
using successively the union bound and Lemma \ref{lem:proba}.
Integrating over $[0,+\infty)$, we obtain that:
\begin{equation}
\mathbb{E}\left[\mu(\widehat{H}_{B,\widehat{m}})-\widetilde{U}_B(\widehat{H}_{B,\widehat{m}})-\text{pen}(B,\widehat{m})   \right]\leq \mathcal{M}\frac{\sqrt{2\pi(B+n)}}{B}.
\end{equation}

Considering now the second term, notice that
\begin{multline*}
\mathbb{E}\left[\inf_{j\geq 1}\left\{\widetilde{U}_B(\widehat{H}_{B,j})+\text{pen}(B,j)  \right\}   -\mu^*_m \right]\leq \mathbb{E}\left[\widetilde{U}_B(\widehat{H}_{B,m})+\text{pen}(B,m)   -\mu^*_m \right]\leq \text{pen}(B,m) .
\end{multline*}

Combining the bounds, we obtain that: $\forall m\geq 1$,
$$
\mathbb{E}\left[\mu(\widehat{H}_{B,\widehat{m}})\right]\leq \mu_m^*+\text{pen}(B,m)+\mathcal{M}\frac{\sqrt{2\pi(B+n)}}{B}.
$$
The oracle inequality is thus proved.

\section{Proof of Theorem \ref{thm:fast}}
We start with proving the following intermediary result, based on the $U$-statistic version of the Bernstein exponential inequality.
\begin{lemma} Suppose that the assumptions of Theorem \ref{thm:fast} are fulfilled. Then, for all $\delta\in (0,1)$, we have with probability larger than $1-\delta$: $\forall r\in \mathcal{R}$, $\forall n\geq 2$,
\begin{equation*}
0\leq \Lambda_n(r)-\Lambda(r)+\sqrt{\frac{2c\Lambda(r)^{\alpha}\log(\#\mathcal{R}/\delta)}{n}}+\frac{4\log(\#\mathcal{R}/\delta)}{3n}.
\end{equation*}
\end{lemma}
\begin{proof}
The proof is a straightforward application of Theorem A on p. 201 in \cite{Ser80}, combined with the union bound and Assumption \ref{assump:noise}.
\end{proof}
The same argument as that used to prove Assertion $(i)$ in Theorem \ref{thm:main} (namely, freezing the $\mathbf{X}_I$'s, applying Hoeffding inequality and the union bound) shows that, for all $\delta\in (0,1)$, we have with probability at least $1-\delta$: $\forall r\in \mathcal{R}$,
$$
0\leq \widetilde{U}_B(q_r)-U_{\mathbf{n}}(q_r)+ \sqrt{\frac{M+\log(M/\delta)}{B}}
$$
for all $n\geq 2$ and $B\geq 1$ (observe that $\mathcal{M}_{\mathcal{H}}\leq 1$ in this case). Now, combining this bound with the previous one and using the union bound, one gets that, for all $\delta\in (0,1)$, we have with probability larger than $1-\delta$: $\forall r\in \mathcal{R}$, $\forall n\geq 2$, $\forall B\geq 1$,
\begin{equation*}
0\leq \widetilde{U}_B(q_r)-\Lambda(r)+\sqrt{\frac{2c\Lambda(r)^{\alpha}\log(2M/\delta)}{n}}+\frac{4\log(2M/\delta)}{3n}+ \sqrt{\frac{M+\log(2M/\delta)}{B}}.
\end{equation*}
Observing that, $\widetilde{U}_B(q_{\widetilde{r}_B})\leq 0$ by definition, we thus have with probability at least $1-\delta$:
$$
\Lambda(\widetilde{r}_B)\leq \sqrt{\frac{2c\Lambda(\widetilde{r}_B)^{\alpha}\log(2M/\delta)}{n}}+\frac{4\log(2M/\delta)}{3n}+ \sqrt{\frac{M+\log(2M/\delta)}{B}}.
$$
Choosing finally $B=O(n^{2/(2-\alpha)})$, the desired result is obtained by solving the inequality above for $\Lambda(\widetilde{r}_B)$.

\section{Proof of Theorem \ref{thm:main2}}

As shown by the following lemma, which is a slight modification of Lemma 1 in \cite{Janson84}, the deviation between the incomplete $U$-statistic and its complete version is of order $O_{\mathbb{P}})(1/\sqrt{B})$ for both sampling schemes.

\begin{lemma} Suppose that the assumptions of \ref{thm:main2} are fulfilled. Then, we have: $\forall H\in \mathcal{H}$,
\begin{equation*}
\mathbb{E}\left[ \left( \bar{U}_{HT}(H)-U_{\mathbf{n}}(H) \right)^2\mid (\mathbf{X}_{I})_{I\in \Lambda} \right]\leq 2\mathcal{M}_{\mathcal{H}}^2/B.
\end{equation*}
\end{lemma}
\begin{proof}
Observe first that, in both cases (sampling without replacement and Bernoulli sampling), we have: $\forall I\neq J$ in $\Lambda$,
$$
\mathbb{E}\left[\left(\Delta(I)-\frac{B}{\#\Lambda}\right)^2\right]\leq \frac{B}{\#\Lambda} \text{ and } \mathbb{E}\left[\left(\Delta(I)-\frac{B}{\#\Lambda}\right)\left(\Delta(J)-\frac{B}{\#\Lambda}\right)\right]\leq \frac{1}{\#\Lambda}\cdot \frac{B}{\#\Lambda}.
$$
Hence, as $(\Delta(I))_{I\in \Lambda}$ and $(\mathbf{X}_{I})_{I\in \Lambda}$ are independent by assumption, we have:
\begin{multline*}
B^2 \mathbb{E}\left[ \left( \bar{U}_{HT}(H)-U_{\mathbf{n}(H)} \right)^2\mid (\mathbf{X}_{I})_{I\in \Lambda} \right]=\mathbb{E}\left[\left( \sum_{I\in \Lambda}\left(\Delta(I)-\frac{B}{\#\Lambda}  \right)H(\mathbf{X}_I) \right)^2 \mid (\mathbf{X}_{I})_{I\in \Lambda} \right]\\
\leq \mathcal{M}_{\mathcal{H}}^2 \sum_{I\in \Lambda}\mathbb{E}\left[\left(\Delta(I)-\frac{B}{\#\Lambda}\right)^2\right]+\mathcal{M}_{\mathcal{H}}^2\sum_{I\neq J}\mathbb{E}\left[\left(\Delta(I)-\frac{B}{\#\Lambda}\right)\left(\Delta(J)-\frac{B}{\#\Lambda}\right)\right]\leq 2B \mathcal{M}_{\mathcal{H}}^2.
\end{multline*}
\end{proof}
Consider first the case of Bernoulli sampling. By virtue of Bernstein inequality applied to the independent variables $(\Delta(I)-B/\#\Lambda)H(\mathbf{X}_{I})$ conditioned upon $(\mathbf{X}_{I})_{I\in \Lambda}$, we have: $\forall H\in \mathcal{H}$, $\forall t>0$,
$$
\mathbb{P}\left\{ \left\vert  \sum_{I\in \Lambda}(\Delta(I)-B/\#\Lambda)H(\mathbf{X}_{I}) \right\vert>t \mid (\mathbf{X}_{I})_{I\in \Lambda} \right\}\leq 2 \exp\left( -\frac{t^2}{4B\mathcal{M}_{\mathcal{H}}^2+2\mathcal{M}_{\mathcal{H}}t/3} \right).
$$
Hence, combining this bound and the union bound, we obtain that: $\forall t>0$,
$$
\mathbb{P}\left\{\sup_{H\in \mathcal{H}}\left\vert \bar{U}_{HT}(H)-U_{\mathbf{n}(H)}  \right\vert >t \mid (\mathbf{X}_{I})_{I\in \Lambda}\right\} \leq 2(1+\#\Lambda)^V \exp\left( -\frac{Bt^2}{4\mathcal{M}_{\mathcal{H}}^2+2\mathcal{M}_{\mathcal{H}}t/3} \right).
$$
Solving $$\delta= 2(1+\#\Lambda)^V \exp\left( -\frac{Bt^2}{4\mathcal{M}_{\mathcal{H}}^2+2\mathcal{M}_{\mathcal{H}}t/3} \right)$$ yields the desired bound.

Consider next the case of the sampling without replacement scheme. Using the exponential inequality tailored to this situation proved in \cite{Serfling74} (see Corollary 1.1 therein), we obtain: $\forall H\in \mathcal{H}$, $\forall t>0$,
$$
\mathbb{P}\left\{\frac{1}{B} \left\vert  \sum_{I\in \Lambda}(\Delta(I)-B/\#\Lambda)H(\mathbf{X}_{I}) \right\vert>t \mid (\mathbf{X}_{I})_{I\in \Lambda} \right\}\leq 2 \exp\left( -\frac{Bt^2}{2\mathcal{M}_{\mathcal{H}}^2} \right).
$$
The proof can be then ended using the union bound, just like above.
\section{Proof of Proposition \ref{prop:var_comp}}

For simplicity, we focus on one sample $U$-statistics of degree two ($K=1$, $d_1=2$) since the argument easily extends to the general case.
Let $U_n(H)$ be a non-degenerate $U$-statistic of degree two:
$$U_n(H) = \frac{2}{n(n-1)} \sum_{i<j} H(x_i,x_j).$$
In order to express the variance of $U_n(H)$ based on its second Hoeffding decomposition (see Section~\ref{subsec:Ustat}), we first introduce more notations: $\forall (x,x')\in \mathcal{X}_1^2$,
$$
H_1(x) \overset{def}{=} \mathbb{E}\left[ H(x,X) \right] - \mu(H) \text{ and }
H_2(x,x') \overset{def}{=} H(x,x') - \mu(H) - H_1(x) - H_1(x').$$
Equipped with these notations, the (orthogonal) Hoeffding/Hajek decomposition of $U_n(H)$ can be written as 
$$U_n(H) = \mu(H) + 2T_n(H) + W_n(H),$$ involving centered and decorrelated random variables given by
\begin{eqnarray*}
T_n(H) &=& \frac{1}{n}\sum_{i=1}^n H_1(x_i),\\
W_n(H) &=& \frac{2}{n(n-1)} \sum_{i<j} H_2(x_i,x_j).
\end{eqnarray*}
Recall that the $U$-statistic $W_n(H)$ is said to be degenerate, since $\mathbb{E}[H_2(x,X)] = 0$ for all $x\in \mathcal{X}_1$.
Based on this representation and setting $\sigma_1^2 = \var[H_1(X)]$ and $\sigma_2^2 = \var[H_2(X,X')]$, the variance of $U_n(H)$ is given by
\begin{equation}
\label{eq:ustatvar}
\var[U_n(H)] = \frac{4\sigma_1^2}{n} + \frac{2\sigma_2^2}{n(n-1)}.
\end{equation}
As already pointed out in Section~\ref{subsec:approx}, the variance of the incomplete $U$-statistic built by sampling with replacement is
\begin{eqnarray}
\label{eq:incomp_var}
\var[\widetilde{U}_{B}(H)] &=& \var[U_n(H)] + \frac{1}{B}\left(1-\frac{2}{n(n-1)}\right)\var[H(X,X')]\nonumber\\
&=& \var[U_n(H)] + \frac{1}{B}\left(1-\frac{2}{n(n-1)}\right)(2\sigma_1^2+\sigma_2^2).
\end{eqnarray}

Take $B = n'(n'-1)$ for $n'\ll n$. It follows from \eqref{eq:ustatvar} and \eqref{eq:incomp_var} that in the asymptotic framework \eqref{asymptotics}, the quantities $\var[U_{n'}(H)]$ and $\var[\widetilde{U}_{B}(H)]$ are of the order $O(1/n')$ and $O(1/n'^2)$ respectively as $n'\rightarrow+\infty$. Hence these convergence rates hold for $\widetilde{g}_{\mathbf{n}'}(\theta)$ and $\widetilde{g}_B(\theta)$ respectively.

  \bibliographystyle{plainnat}
\bibliography{mvset}

\end{document}